\documentclass[twoside]{article}

\usepackage[accepted]{aistats2022}
\usepackage{hyperref}       
\hypersetup{
    colorlinks,
    breaklinks,
    linkcolor = blue,
    citecolor = black,
    urlcolor  = black,
}
\usepackage{bbm}

\def \tree {\textsf{Tree}}

\usepackage{microtype}
\usepackage{nicefrac}
\usepackage{tabularx}
\usepackage{multirow}
\usepackage{amsmath}
\usepackage{amssymb}
\usepackage{amsthm}
\usepackage{bbm}
\usepackage{graphicx}
\usepackage{subcaption}
\usepackage{algorithm}
\usepackage{algpseudocode}
\usepackage[round]{natbib}

\theoremstyle{definition}

\newtheorem{theorem}{Theorem}
\newtheorem{definition}{Definition}

\newtheorem{lemma}{Lemma}
\newtheorem{remark}{Remark}

\def\lc{\left\lceil}
\def\rc{\right\rceil}

\renewcommand{\tilde}{\widetilde}

\newcommand*\diff{\mathop{}\!\mathrm{d}}

\usepackage{amsthm}

\newcommand{\BlackBox}{\rule{1.5ex}{1.5ex}}  
\newenvironment{proof}{\par\noindent{\bf Proof\ }}{\hfill\BlackBox\\[2mm]}

\begin{document}

%
\runningtitle{Optimal Rates of (Locally) Differentially Private Heavy-tailed Multi-Armed Bandits}

%

\twocolumn[

\aistatstitle{Optimal Rates of (Locally) Differentially Private Heavy-tailed \\
Multi-Armed Bandits}

\aistatsauthor{ Youming Tao \footnotemark[1]
  \And Yulian Wu \footnotemark[1] \And  Peng Zhao \And Di Wang }

\aistatsaddress{ Shandong University \And  KAUST \And Nanjing University \And KAUST } ]

\footnotetext[1]{The first two authors contributed equally. Part of the work was done when Youming Tao was a research intern at KAUST, and Yulian Wu was a student at East China Normal University.}

\begin{abstract}
  In this paper we investigate the problem of stochastic multi-armed bandits (MAB) in the (local) differential privacy (DP/LDP) model. Unlike previous results that assume bounded/sub-Gaussian reward distributions, we focus on the setting where each arm's reward distribution only has $(1+v)$-th moment with some $v\in (0, 1]$. In the first part, we study the problem in the central $\epsilon$-DP model. We first provide a near-optimal result by developing a private and robust Upper Confidence Bound (UCB) algorithm. Then, we improve the result via a private and robust version of the Successive Elimination (SE) algorithm. Finally, we establish the lower bound to show that the instance-dependent regret of our improved algorithm is optimal. In the second part, we study the problem in the $\epsilon$-LDP model. We propose an algorithm that can be seen as locally private and robust version of SE algorithm, which provably achieves (near) optimal rates for both instance-dependent and instance-independent regret. Our results reveal differences between the problem of private MAB with bounded/sub-Gaussian rewards and heavy-tailed rewards. To achieve these (near) optimal rates, we develop several new hard instances and private robust estimators as byproducts, which might be used to other related problems. Finally, experiments also support our theoretical findings and show the effectiveness of our algorithms. 
\end{abstract}

\section{INTRODUCTION}
As one of the most fundamental problems in statistics and machine learning, (stochastic) Multi-Armed Bandits (MAB), and its general form, bandit learning, have already been studied for more than half a century, starting from \citet{thompson1933likelihood} and~\citet{robbins1952some}. They find numerous applications in many areas such as medicine \citep{gutierrez2017multi}, finance \citep{shen2015portfolio}, social science \citep{nakayama2017nash}, and clinical research \citep{press2009bandit}. The wide applications of bandit learning also present some new challenges to existing methods. Particularly, due to the existence of sensitive data  and their  distributed nature in many applications like recommendation system, biomedicine, and genomics, it is often challenging to preserve the privacy of such data, which makes the data
extremely difficult to aggregate and learn from.

To preserve the privacy of these sensitive data, Differential Privacy (DP) \citep{dwork2006calibrating} has received a great deal of attention and now has established itself as a de facto notation of privacy for data analysis. Over the past decade, differentially private bandit learning has been extensively studied from various setups including classical stochastic MAB~\citep{mishra2015nearly,tossou2016algorithms,sajed2019optimal,ren2020multi,kalogerias2020best}, combinatorial semi-bandits~\citep{chen2020locally}, and contextual bandits~\citep{shariff2018differentially,hannun2019privacy,malekzadeh2020privacy,zheng2020locally}. Additionally, \citet{wang2020global,dubey2020differentially,dubey2020private} recently investigated bandit learning in the federated/distributed setting.

\begin{table*}[htbp]
\caption{Summary of our contributions and comparison with the bounded/sub-Gaussian reward distribution. All the results are in the expected regret form. For the heavy-tailed reward distribution case, we assume the $(1+v)$-th moment of each reward distribution is bounded by $1$ for some known $v\in (0, 1]$. For the bounded reward distribution case, we assume the rewards are bounded by 1.  For the sub-Gaussian reward distribution case, we assume the variance of each reward distributed is bounded by $1$. Here $K$ is the number of arms, $T$ is the number of rounds, and $\Delta_a$ is the mean reward gap of arm $a$.}

\begin{center}
\resizebox{\textwidth}{!}{	
\begin{tabular}{|c|l|l|l|}
\hline
Problem                              & Model          & Upper Bound                                                                                                      & Lower Bound                                                                                                                                                                                                                                                                                                            \\ \hline
\multirow{2}{*}{\begin{tabular}[c]{@{}c@{}}Heavy-tailed Reward\\ (Instance-dependent Bound)\end{tabular}}   & $\epsilon$-DP  & $ O\left(\frac{\log T}{\epsilon}\sum_{\Delta_a>0}{(\frac{1}{\Delta_a})^{\frac{1}{v}}}+\max_a\Delta_a\right)$     & $\Omega \left(\frac{\log T}{\epsilon}\sum_{\Delta_a>0}{(\frac{1}{\Delta_a})^{\frac{1}{v}}}\right) $                    \\ \cline{2-4}
                                     & $\epsilon$-LDP & ${O} \left(\frac{\log T}{\epsilon^2}\sum_{\Delta_a>0}{(\frac{1}{\Delta_a})^{\frac{1}{v}}}+\max_a\Delta_a\right)$ & {$\Omega \left(\frac{\log T}{\epsilon^2}{\sum_{\Delta_a>0}(\frac{1}{\Delta_a})^{\frac{1}{v}}}\right) $                                 }                                                                                                            \\ \hline
\multirow{2}{*}{\begin{tabular}[c]{@{}c@{}}Bounded/sub-Gaussian Reward\\ (Instance-dependent Bound)\end{tabular}}      & $\epsilon$-DP  & $O\left(\frac{K\log T}{\epsilon}+\sum_{\Delta_a>0} \frac{\log T}{\Delta_a}\right)$ \citep{sajed2019optimal}                                                                                                                &         $\Omega\left(\frac{K\log T}{\epsilon}+\sum_{\Delta_a>0} \frac{\log T}{\Delta_a}\right)$                              \citep{shariff2018differentially}                                                                                                                                                                                                                                                                               \\ \cline{2-4}
                                     & $\epsilon$-LDP &     $O\left(\frac{1}{\epsilon^2}\sum_{\Delta_a>0} \frac{\log T}{\Delta_a}+\Delta_a\right)$ \citep{ren2020multi}                                                                                                            &      $\Omega\left(\frac{1}{\epsilon^2}\sum_{\Delta_a>0} \frac{\log T}{\Delta_a}\right)$  \citep{ren2020multi}                                                                                                                                                                                                                                                                                                                                            \\ \hline

\multirow{2}{*}{\begin{tabular}[c]{@{}c@{}}Heavy-tailed Reward\\ (Instance-independent Bound)\end{tabular}}     & $\epsilon$-DP  &     $O\left(\left(\frac{K\log T}{\epsilon}\right)^{\frac{v}{1+v}}T^{\frac{1}{1+v}}\right)$                                                                                                                &               -----                                                                                                                                                                                                                                                                                                                                \\ \cline{2-4}
                                     & $\epsilon$-LDP &                                                                                                  $O\left(\left(\frac{K\log T}{\epsilon^2}\right)^{\frac{v}{1+v}}T^{\frac{1}{1+v}}\right)$      &         $\Omega\left(\left(\frac{K}{\epsilon^2}\right)^{\frac{v}{1+v}}T^{\frac{1}{1+v}}\right)$                                                                                                                                                                                                                                                                                                                                      \\ \hline

\multirow{2}{*}{\begin{tabular}[c]{@{}c@{}}Bounded/sub-Gaussian Reward\\ (Instance-independent Bound)\end{tabular}}     & $\epsilon$-DP  &     $O\left(\sqrt{KT\log T}+ \frac{K\log T}{\epsilon}\right)$     \citep{sajed2019optimal}                                                                                                            &                  $\Omega\left(\sqrt{KT}+ \frac{K\log T}{\epsilon}\right)$ \citep{sajed2019optimal}                                                                                                                                                                                                                                                                                                                 \\ \cline{2-4}
                                     & $\epsilon$-LDP &      $O\left(\frac{\sqrt{KT\log T}}{\epsilon}\right)$ \citep{ren2020multi}                                                                                                            &          $\Omega(\frac{\sqrt{KT}}{\epsilon})$ \citep{basu2019differential}                                                                                                                                                                                                                                                                                                                                   \\ \hline
\end{tabular}
}
\label{Table:1}
\end{center}
\end{table*}
However, these problems are still not well-understood. For example, all of the previous results and methods need to assume that the rewards are sampled from some bounded (or sub-Gaussian) distributions to guarantee the DP property. However, such assumptions may not hold when designing decision-making algorithms for complicated real-world systems. In particular, previous papers have shown that the rewards or the interactions in such systems often lead to heavy-tailed and power law distributions \citep{dubey2019thompson}, such as modeling stock prices \citep{bradley2003financial}, preferential attachment in social networks \citep{mahanti2013tale}, and online behavior on websites \citep{kumar2010characterization}. Thus, it is necessary to develop new methods to deal with these heavy-tailed rewards in the private bandit learning.

To address the above issue, in this paper, we focus on the most fundamental bandit model, {\em i.e.,} multi-armed bandits, with heavy-tailed rewards.  We conduct a comprehensive and the first study on MAB with heavy-tailed rewards in both central and local DP models, where the reward distribution of each arm only has the $(1+v)$-th moment for some $v\in (0, 1]$. Our contributions are summarized as follows.
\begin{itemize}
    \item In the first part (Section~\ref{sec:central}), we consider the problem in the central $\epsilon$-DP model. Specifically, we first propose a method based on a robust version of the Upper Confidence Bound (UCB) algorithm, and also design a new mechanism that could be seen as an adaptive version of the Tree-based mechanism~\citep{dwork2010differential}. To further improve the result, we then develop a private and robust version of the Successive Elimination (SE) algorithm and show that the (expected) regret bound is improved by a factor of $\log^{1.5+\frac{1}{v}} T$, where $T$ is the number of rounds. Moreover, we establish the lower bound and show that the instance-dependent regret bound of $ O\left(\frac{\log T}{\epsilon}\sum_{\Delta_a>0}{(\frac{1}{\Delta_a})^{\frac{1}{v}}}+\max_a\Delta_a\right)$ achieved by our second algorithm is optimal (up to $\text{poly}(\log\log \frac{1}{\Delta_a})$ factors), where $\Delta_a$ is the mean gap defined in Section~\ref{sec:3.1}.
    \item In the second part (Section~\ref{sec:LDP}), we study the problem in the $\epsilon$-LDP model. We first develop a LDP version of the SE algorithm which achieves an instance-dependent regret bound of ${O} \left(\frac{\log T}{\epsilon^2}\sum_{\Delta_a>0}{(\frac{1}{\Delta_a})^{\frac{1}{v}}}+\max_a\Delta_a\right)$ and an $\tilde{O}\left(\left(\frac{K}{\epsilon^2}\right)^{\frac{v}{1+v}}T^{\frac{1}{1+v}}\right)$ instance-independent bound. Then, we show that the above instance-dependent regret bound is optimal and the instance-independent regret bound is near-optimal (up to $\text{poly}(\log T)$ factors).
    \item All of our results also reveal the differences between the problem of private MAB with bounded/sub-Guassian rewards and that with heavy-tailed rewards (see Table \ref{Table:1} for details).  To achieve these (near) optimal results, we develop several new hard instances, mechanisms and private robust estimators as byproducts, which could be used to other related problems, such as private contextual bandits~\citep{shariff2018differentially} or private reinforcement learning~\citep{vietri2020private}.
    \end{itemize}

Due to space limitation, all the technical lemmas and proofs are included in the appendix. The source code is also included in Supplementary Materials.

\section{RELATED WORK}
As mentioned earlier, there are enormous previous works on either MAB with bounded/sub-Gaussian reward distributions in  the (local) DP model \citep{mishra2015nearly,tossou2016algorithms,gajane2018corrupt,shariff2018differentially,basu2019differential,sajed2019optimal,ren2020multi,vietri2020private,zheng2020locally} or MAB with heavy-tailed reward distributions \citep{bubeck2013bandits,lee2020optimal,yu2018pure,lattimore2017scale,agrawal2021regret,vakili2013deterministic,agrawal2020optimal}. However, to the best of our knowledge, MAB with heavy-tailed reward in the (local) DP model has not been studied before. In the following we only discuss previous works that are the most close to ours.

In the previous studies of MAB with bounded/sub-Gaussian rewards, to guarantee DP property, the most direct way is to modify the classical UCB algorithm~\citep{mishra2015nearly,tossou2016algorithms,ren2020multi}. Our first algorithm is motivated by a robust version of the UCB algorithm in \citep{bubeck2013bandits}. However, there are several differences. First, unlike the non-private setting, in this paper we show that this approach could only achieve a suboptimal instance-dependent (expected) regret bound, due to the noises added by the Tree-based mechanism \citep{chan2011private}. Secondly, due to the added noises, parameters such as the thresholds are quite different with the non-private case. To achieve an improved regret bound, our second algorithm is based on the Successive Elimination (SE) \citep{even2006action} algorithm, whose private version has been studied in \citep{sajed2019optimal} for the bounded reward case. In this paper we extend the method to the heavy-tailed case and the LDP model. Our algorithms are provably (near) optimal. For the lower bounds, previous papers established hard instances for either private MAB with bounded rewards \citep{basu2019differential,sajed2019optimal,ren2020multi} or heavy-tailed MAB in the non-private case \citep{bubeck2013bandits}. However, these instances cannot be used to our problem and this paper builds new hard instances.

Private and robust estimation has drawn much attention in recent years. \citet{barber2014privacy} provided the first study on private mean estimation for distributions with bounded moment, which is extended by \citet{kamath2020private,brunel2020propose,liu2021robust} recently. However, all of them need to assume the underlying distribution has the second-order moment, while in this paper we only need to assume the reward distributions have the $(1+v)$-th moment for some $v\in (0, 1]$. Moreover, all of these works only focus on the central DP model and offline setting, and it is generally unclear whether they could be extended to the stream setting. Thus, our problem is more general. In addition to the mean estimation problem, recently \citet{wang2020differentially} studied differentially private stochastic convex optimization with heavy-tailed data, while  their work still requires to assume the distribution of gradient has second-order moment and cannot be used to the stream setting.

\section{PRELIMINARIES}

In this section, we present some preliminaries for MAB with heavy-tailed rewards and differential privacy.

\subsection{MAB with Heavy-tailed Rewards} \label{sec:3.1}
In a stochastic multi-armed bandits (MAB) problem, there is a learner interacting with the environment sequentially over $T$ rounds. The learner is faced with a set of $K$ independent arms $\{1,\ldots, K\}$. In each round $t\in[T]$, the learner selects an arm $a_t \in[K]$ to pull and then obtains a reward $x_{t}$ drawn i.i.d. from a fixed but unknown probability distribution $\mathcal{X}_{a_t}$ associated with the chosen arm. Denote by $\mu_a$ the mean of each distribution $\mathcal{X}_a$ for $a\in [K]$, and by  $\mu^*=\max_{a\in [K]}\mu_a$ the maximum. Define $\Delta_a\triangleq\mu^*-\mu_a$ as the mean reward gap for arm $a$. The learner aims to maximize her/his expected cumulative reward over time, {\em i.e.,} to minimize the (expected) cumulative \textit{regret}, defined as
\begin{equation}
    \mathcal{R}_T\triangleq T\mu^*-\mathbb{E}\left[\sum\limits_{t=1}^{T}{x_t}\right],
\end{equation}
where the expectation is taken with respect to all the randomness. This paper considers a heavy-tailed setting where each arm's reward distribution only has finite raw moments of order $1+v$ for some $v\in(0,1]$. Concretely, we assume that there is a constant $u>0$ such that for each reward distribution $\mathcal{X}_a$,
\begin{equation}\label{eq:2}
    \mathbb{E}_{X\sim \mathcal{X}_a} [|X|^{1+v}] \leq u.
\end{equation}
In this paper, we assume both $v$ and $u$ are known constants, {\em i.e.,} for any constant $c$ we regard $c^\frac{1}{v}$ as a constant. Note that the assumption is commonly used in robust estimation \citep{catoni2012challenging,kamath2020private,wang2020differentially} and MAB with heavy-tailed rewards \citep{bubeck2013bandits,dubey2019thompson,lee2020optimal,agrawal2021regret}.

Instead of the assumption on the raw moment in (\ref{eq:2}), there is another assumption on the central moment instead, {\em i.e., $ \mathbb{E}_{X\sim \mathcal{X}_a} [|X-\mathbb{E}(X)|^{1+v}] \leq u$}. We note that both of the raw moment and central moment assumptions have been studied in the previous work on private robust estimation \citep{wang2020differentially,kamath2020private} for the mean estimation of distributions with bounded second-order moment. Here, we claim that, the bounded raw moment implies that the bounded central moment, and vice versa. See Lemma~\ref{lemma-raw-central} in Appendix for details.
\subsection{Differential Privacy}
We introduce the definition of differential privacy (DP) in the stream setting since rewards are released continually.
According to \citep{dwork2010differential}, for data streams there are two different settings, {\em i.e.,}  event-level setting and user-level setting. This paper will focus on the event-level setting, {\em i.e.,} two data streams $\sigma$ and $\sigma^\prime$ are adjacent if they differ at exactly one timestep. Intuitively, an algorithm is differentially private  if it cannot be used to distinguish any two adjacent streams.
\begin{definition}[Differential Privacy \citep{dwork2010differential}]
   An algorithm $\mathcal{M}$ is $\epsilon$-differentially private (DP) if for any adjacent streams $\sigma$ and $\sigma^\prime$, and any measurable subset $\mathcal{O}$ of the output space of $\mathcal{M}$, we have
       $\mathbb{P}\left[\mathcal{M}(\sigma)\in\mathcal{O}\right]
       \le e^\epsilon\cdot\mathbb{P}\left[\mathcal{M}(\sigma^\prime)\in\mathcal{O}\right].$
\end{definition}

Compared with DP, Local Differential Privacy~(LDP) is a stronger notion of privacy. In LDP, each data is perturbed before collection to ensure privacy.

\begin{definition}[Local Differential Privacy]
   An algorithm $\mathcal{M}:\mathcal{X}\to\mathcal{Y}$ is said to be $\epsilon$-locally differentially private (LDP) if for any $x, x^\prime\in\mathcal{X}$, and any measurable subset $\mathcal{O}\subset\mathcal{Y}$, it holds that
$ \mathbb{P}\left[\mathcal{M}(x)\in\mathcal{O}\right]
       \le e^\epsilon\cdot\mathbb{P}\left[\mathcal{M}(x^\prime)\in\mathcal{O}\right].
 $
\end{definition}
In this paper, we will mainly use the Laplacian and an adaptive version of the Tree-based mechanism (see Section \ref{sec:central} for details), and the parallel composition theorem to guarantee the DP property.
\begin{lemma}[Parallel Composition]
    Suppose there are $n$ $\epsilon$-differentially private mechanisms $\{\mathcal{M}_i\}_{i=1}^n$ and $n$ disjoint datasets denoted by $\{D_i\}_{i=1}^n$. Then the algorithm, which applies each $\mathcal{M}_i$ on the corresponding $D_i$, preserves $\epsilon$-DP in total.
\end{lemma}
\begin{definition}[Laplacian Mechanism]\label{def:3}
		Given a function $f : \mathcal{X}^n\rightarrow \mathbb{R}^d$, the Laplacian Mechanism is defined as:
		$\mathcal{M}_L(D,f,\epsilon)=f(D)+ (Y_1, Y_2, \cdots, Y_d),$
		where $Y_i$ is i.i.d. drawn from a Laplacian distribution $\text{Lap}(\frac{\Delta_1(f)}{\epsilon}),$ where $\Delta_1(f)$ is the $\ell_1$-sensitivity of the function $f$, {\em i.e.,}
		$\Delta_1(f)=\sup_{D\sim D^\prime}||f(D)-f(D')||_1$. Here, $D\sim D^\prime$ denotes that $D$ and $D^\prime$ are neighbouring datasets, \emph{i.e.}, those that differ in exactly on entry. For a parameter $\lambda$, the Laplacian distribution has the density function $\text{Lap}(\lambda) (x)=\frac{1}{2\lambda}\exp(-\frac{|x|}{\lambda})$.
		Laplacian Mechanism preserves $\epsilon$-DP.
\end{definition}

\section{DP HEAVY-TAILED MAB}\label{sec:central}
In this section, we will study the problem of designing $\epsilon$-DP algorithms for MAB with heavy-tailed rewards. Recall that, in the classical setting where the rewards follow some bounded distributions, the most commonly used approach is using the Tree-based mechanism to privately calculate the sum of rewards and then modify the Upper Confidence Bound (UCB) algorithm \citep{auer2002finite}, such as \citep{mishra2015nearly,tossou2016algorithms}. However, their methods cannot be directly generalized to the heavy-tailed setting, since now the reward is unbounded. Thus, the most natural idea is to first preprocess the rewards to make them bounded and then use the Tree-based mechanism and UCB algorithm.

To address MAB with heavy-tailed in the non-private case, \citet{bubeck2013bandits} presented a general near-optimal framework called robust-UCB by combining the UCB algorithm with several robust mean estimators. Specifically, the framework first truncates the rewards to some bounded value and then performs a robust version of UCB. Building upon the framework, we first design a method for DOP heavy-tailed MAB based on the above non-private robust-UCB algorithm, see Algorithm \ref{Algo-UCB} for details.

\begin{algorithm}[!t]
    \caption{DP Robust Upper Confidence Bound}
    \label{Algo-UCB}
    \begin{algorithmic}[1]
        \Require time horizon $T$, parameters $\epsilon, v,u$.
        \State Create an empty tree $\tree_{a}$ for each arm $a\in[K]$.
        \State Initialize pull number $n_a\gets 0$ for each arm $a\in[K]$.
        \State Denote $B_{n}$ as $(\frac{\epsilon un}{\log^{1.5}T})^{1/(1+v)}$ for any $n\in\mathbb{N}^+$.
        \For {$t= 1,\ldots,K$}
            \State Pull arm $t$ and observe a reward $x_t$.
            \State Update the pull number $n_t\gets n_t+1$.
            \State Truncate the reward by $\widetilde{x}_t\gets x_t\cdot\mathbb{I}_{| x_t|\le B_{n_t}}$.
            \State Insert $\widetilde{x}_t$ into $\tree_t$.
        \EndFor
        \For {$t= K+1,\ldots,T$}
            \State Obtain $\widehat{S}_a(t)$ for each $a\in[K]$ via Algorithm \ref{Algo-Tree}.
            \State Pull arm $$a_t=\mathop{\arg\max}_{a} \frac{\widehat{S}_a(t)}{n_a}+18u^{\frac{1}{1+v}}(\frac{\log(2t^4)\log^{1.5+\frac{1}{v}}T}{n_a\epsilon})^{\frac{v}{1+v}}$$ and observe the reward $x_{t}$.
            \State Update the pull number $n_{a_t}\gets n_{a_t}+1$.
            \State Truncate the reward by $\widetilde{x}_{t}\gets x_{t}\cdot\mathbb{I}_{|x_{t}|\le B_{n_{a_t}}}$.
            \State Insert $\widetilde{x}_t$ into $\tree_{a_t}$.
        \EndFor
    \end{algorithmic}
\end{algorithm}


The key idea of our algorithm is that, in the first $K$ rounds, we establish a tree instance $\tree_a$ for each arm $a\in[K]$ (step 4-9). After that, at round $t$, when the arm $a_t$ is pulled, we truncate the newly generated reward by a certain range $B_{n_{a_{t}}}$ and insert the truncated reward to $\tree_{a_t}$ (step 10-15). Here we use a robust version of UCB to select the arm, where the sum of rewards is given by the Tree-based mechanism (since we only insert the truncated rewards, we can use the mechanism). We note that in the original Tree-based mechanism in \citep{chan2011private,dwork2010differential} each element in the data steam is bounded by a uniform constant $B$. However, here the bound $B_{n_{a_{t}}}$ is adaptive and non-decreasing. Thus, we need a finer tree mechanism. To this end, we propose an adaptive Tree-based mechanism based on the earlier works~\citep{chan2011private,DworkNRRV_STOC09}, whose procedures are presented in Algorithm~\ref{Algo-Tree}. By the same proof as in \citep{chan2011private,DworkNRRV_STOC09} we have the following guarantees as shown in Lemma~\ref{lemma-tree}.


\begin{definition}[p-sum]\label{def-psum}
        A p-sum is a partial sum of consecutive data items. Let $1\le i\le j$. For a data stream $\sigma$ of length $T$, we use $\sigma(t)$ to denote the data item at time $t\in[T]$ and $\sum[i,j]\triangleq\sum_{k=i}^j\sigma(k)$ to denote a partial sum involving data items $i$ through $j$. We use the notation $\alpha_i^t$ to denote the p-sum $\sum[t-2^i+1,t]$.
\end{definition}

\begin{lemma}[{(Adaptive) Tree-based Mechanism}]
\label{lemma-tree}
Given a stream $\sigma$ such that $\sigma(t)\in[-B_t, B_t]$ for $\forall t\in[T]$, where $B_t$ is non-decreasing with $t$, we want to privately and continually release the sum of the stream $S(t)\triangleq\sum\nolimits_{i=1}^{t}{\sigma(i)}$ for each $t\in [T]$. Tree-based Mechanism (Algorithm \ref{Algo-Tree}) outputs an estimation $\widehat{S}(t)$ for $S(t)$ at each $t\in[T]$ such that $\widehat{S}(t)$ preserves $\epsilon$-differential privacy and guarantees the following noise bound with probability at least $1-\delta$ for any $\delta>0$,
\begin{equation}
    \left| \widehat{S}(t) - S(t) \right| \leq  \frac{2B_t}{\epsilon} \cdot\log^{1.5} T \cdot \log\frac{1}{\delta}.
\end{equation}
\end{lemma}
When $B_t=B$, Algorithm \ref{Algo-Tree} will be the same as the original one. Theorem~\ref{theorem-privacy-ucb} presents the privacy guarantee of overall algorithm (Algorithm~\ref{Algo-UCB} and Algorithm~\ref{Algo-Tree}).

\begin{algorithm}[!t]
    \caption{(Adaptive) Tree-based Mechanism}
    \label{Algo-Tree}
    \begin{algorithmic}[1]
        \Require time horizon $T$, privacy budget $\epsilon$, a stream $\sigma$.
        \Ensure A private version $\widehat{S}(t)$ for $S(t)=\sum\nolimits_{i=1}^{t}{\sigma(i)}$ at each $t\in[T]$
        \State Initialize each p-sum $\alpha_i$ and noisy p-sum $\widehat{\alpha}_i$ to $0$.
        \State $\epsilon^\prime\gets\epsilon/\log T$.
        \For {$t = 1,\ldots,T$}
            \State Express $t$ in binary form: $t=\sum_j {\rm Bin}_j(t)\cdot 2^j$.
            \State $i\gets\min\{j:{\rm Bin}_j(t)\neq 0\}$.
            \State $\alpha_i\gets\sum_{j<i}{\alpha_j}+\sigma(t)$.
            \For {$j= 0,\ldots,i-1$}
                \State $\alpha_j\gets0$, $\widehat{\alpha}_j\gets0$.
            \EndFor
            \State $\widehat{\alpha}_i\gets\alpha_i+{\rm Lap}(2B_t/\epsilon^\prime)$.
            \State \Return $\hat{S}(t)\gets\sum_{j:{\rm Bin}_j(t)=1}{\widehat{\alpha}_j}$.
        \EndFor
    \end{algorithmic}
\end{algorithm} 

\begin{theorem} \label{theorem-privacy-ucb}
   For any $\epsilon>0$, the overall algorithm (Algorithm~\ref{Algo-UCB} and Algorithm~\ref{Algo-Tree}) is $\epsilon$-differentially private.
\end{theorem}

In fact, the $\widehat{S}_a(t)/n_a$ term in step 12, which is denoted by $\widehat{\mu}_{a}(n_a, t)$, could be seen as a robust and private estimator of the mean $\mu_a$ after total $n_a$ pulls of arm $a$ till time $t$. Our selection strategy in step 12 is based on the following estimation error between $\widehat{\mu}_{a}(n_a, t)$ and $\mu_a$, which is also a key lemma that will be used to bound the regret of Algorithm~\ref{Algo-UCB}.

\begin{lemma}\label{lemma-esterr}
    In Algorithm \ref{Algo-UCB}, for a fixed arm $a$ and $t$, we have the following estimation error with probability at least $1-t^{-4}$,
    \begin{equation}
        \widehat{\mu}_{a}(n_a, t)\le \mu_a+18u^{\frac{1}{1+v}}
        \left(\frac{\log(2t^4)\log^{1.5+\frac{1}{v}}T}
        {n_a\epsilon}
        \right)^{\frac{v}{1+v}}.
    \end{equation}

\end{lemma}

We have the following instance-dependent regret bound by the proof of Theorem 1 in \citep{auer2002finite}.
\begin{theorem}\label{thm:2}
 Under our assumptions, for any $0<\epsilon\leq 1$
    the instance-dependent expected regret of Algorithm~\ref{Algo-UCB} satisfies
    \begin{equation}
        \mathcal{R}_T\leq O \left(\sum_{a:\Delta_a>0}{\Big( \frac{\log^{2.5+\frac{1}{v}} T}{\epsilon}\Big(\frac{u}{\Delta_a}\Big)^{\frac{1}{v}}+\Delta_a\Big)}\right).
    \end{equation}
\end{theorem}
Compared with the non-private version of robust UCB~\citep{bubeck2013bandits}, the main difference is the threshold value $B_{n_{a_t}}=(\frac{\epsilon un_{a_t}}{\log^{1.5}T})^{\frac{1}{1+v}}$, where \citet{bubeck2013bandits} set it as $(\frac{ un_{a_t}}{\log (t^2)})^{\frac{1}{1+v}}$. Informally speaking, this is caused by the fact that, due to the privacy, the number of efficient samples now becomes $n\epsilon$. Specifically, due to privacy constraint, the estimation error could be decomposed into three parts: the bias, variance due to the truncation, and the noise we added. We can show that setting $B_{n_{a_t}}$ as the threshold could provide an improved bound of error. Compared with the $O(\sum_{a:\Delta_a>0}{[\log T(\frac{u}{\Delta_a})^{\frac{1}{v}}+\Delta_a]})$ optimal rate of the regret in the non-private version~\citep{bubeck2013bandits}, we can see that there is an additional factor of $\frac{\log^{1.5+\frac{1}{v}} T}{\epsilon}$ in the private case. In contrast, in the problem where the reward distributions are bounded, it has been shown by \citet{shariff2018differentially} that, there is only an additional factor of $\frac{1}{\epsilon}$ compared with the non-private case.
Thus, a natural question arises here is \emph{whether it is possible to further improve the regret}.  We answer this question affirmatively by designing an optimal algorithm, see Algorithm~\ref{Algo-DPSE} for details.

Our algorithm is based on the Successive Elimination (SE) algorithm proposed by \citet{even2006action}, whose DP variant has been studied by~\citet{sajed2019optimal}. Briefly speaking, we first set all the arms as viable options (step 1), then in each epoch we pull all the viable arms to get the same (private) confidence interval around their empirical rewards (step 4-18). Finally we eliminate the arms with lower empirical rewards from the viable options if they are sub-optimal compared with other viable arms (step 21-25).


\begin{algorithm}[!t]
    \caption{DP Robust Successive Elimination}
    \label{Algo-DPSE}
    \begin{algorithmic}[1]
        \Require confidence $\beta$, parameters $\epsilon, v,u$.
        \State $\mathcal{S}\gets\{1,\cdots,K\}$
        \State Initialize: $t\gets0$, $\tau\gets0$.
        \Repeat
            \State $\tau\gets \tau+1$.
            \State Set $\bar{\mu}_a=0$ for all $a\in\mathcal{S}$.
            \State $r\gets 0$, $D_\tau\gets2^{-\tau}$.
            \State $ R_\tau\gets \lc u^{\frac{1}{v}}
            (
            \frac
            {
                24^{(1+v)/v}
                \log
                    (
                        4|\mathcal{S} |\tau^2/\beta
                    )
            }
            {
                \epsilon D_\tau^{(1+v)/v}
            }
            )+1 \rc$.
            \State $B_\tau\gets
                    (
                    \frac
                    {
                        uR_\tau\epsilon
                    }
                    {
                        \log( 4|\mathcal{S}|\tau^2/\beta)
                    }
                    )^{1/(1+v)}$.
            \While {$r<R_\tau$}
                \State $r\gets r+1$.
                \For {$a\in\mathcal{S}$}
                    \State $t\gets t+1$.
                    \State Sample a reward $x_{a,r}$.
                \State $\widetilde{x}_{a,r}\gets x_{a,r}\cdot\mathbb{I}_{\{|x_{a,r}|\le B_\tau\}}$.
                \EndFor
            \EndWhile
            \State For each $a\in \mathcal{S}$, compute $\bar{\mu}_a\gets(\sum\limits_{l=1}^{R_\tau}\widetilde{x}_{a,l})/R_\tau$.
            \State Set $\widetilde{\mu}_a\gets\bar{\mu}_a+{\rm Lap}(\frac{2B_\tau}{R_\tau\epsilon})$ for all $a\in\mathcal{S}$.
            \State $\widetilde{\mu}_{\rm max}\gets\max_{a\in\mathcal{S}}\widetilde{\mu}_a$.
            \State $err_\tau\gets u^{1/(1+v)}
            (\frac
            {\log(4|\mathcal{S}|\tau^2/\beta)}
            {R_\tau\epsilon})^{v/(1+v)}$.
            \For {all viable arm $a$}
                \If {$\widetilde{\mu}_{\rm max}-\widetilde{\mu}_a>12err_\tau$} \State Remove arm $a$ from $\mathcal{S}$. \EndIf
            \EndFor
        \Until {$|\mathcal{S}|=1$}
        \State Pull the arm in $\mathcal{S}$ in all remaining $T-t$ rounds.
    \end{algorithmic}
\end{algorithm} 

\begin{theorem}\label{thm:3}
 For any $\epsilon>0$, Algorithm~\ref{Algo-DPSE} is $\epsilon$-differentially private.

\end{theorem}
\begin{remark}
As mentioned earlier, \citet{sajed2019optimal} also studied a DP variant of the SE algorithm. However, there are several differences between their result and ours. Their algorithm is only for bounded reward distributions, while here we focus on the heavy-tailed ones. Due to the irregularity of rewards, we need to preprocess and shrink these rewards. Moreover, the forms of parameters are also more complicated than the bounded distributions case. Finally, in the later section, we also extend Algorithm~\ref{Algo-DPSE} to the local model and show its optimality.
\end{remark}
The following lemma claims that the number of rounds to pull each arm $a$ is at most $\tilde{O}(\frac{1}{\epsilon (\Delta_a)^\frac{1+v}{v}})$.
\begin{lemma}\label{lemma-pulltime}
    For any instance of the $K$-MAB problem, denote by $a^*$ its optimal arm. Fix the time horizon $T$ and confidence level $\beta \in (0,1)$. Then, with probability at least $1-\beta$, in Algorithm~\ref{Algo-DPSE}, the total number of rounds to pull each sub-optimal arm $a\ne a^*$ is at most
    \begin{equation}
        \min \left\{T, O\left( \frac{u^{\frac{1}{1+v}}}{\epsilon(\Delta_a)^{\frac{1+v}{v}}}\left(\log(\frac{K}{\beta})+\log\log(\frac{1}{\Delta_a})\right)\right)\right\}.
    \end{equation}

\end{lemma}

\begin{theorem}[DP Upper Bound]
\label{thm:4}
If we set $\beta=\frac{1}{T}$ in Algorithm~\ref{Algo-DPSE}, then for sufficiently large $T$ and any $\epsilon \in (0,1]$,  the instance-dependent expected regret of Algorithm~\ref{Algo-DPSE} satisfies
    \begin{equation}
          \mathcal{R}_T\leq  {O} \left(\frac{u^{\frac{1}{1+v}}\log T}{\epsilon}\sum_{\Delta_a>0}{\Big(\frac{1}{\Delta_a}\Big)^{\frac{1}{v}}}+\max_a\Delta_a\right).
    \end{equation}
Moreover, the instance-independent expected regret of Algorithm~\ref{Algo-DPSE} satisfies
  \begin{equation}
     \mathcal{R}_T \leq  O\left(u^{\frac{v}{(1+v)^2}}\left(\frac{K\log T}{\epsilon}\right)^{\frac{v}{1+v}}T^{\frac{1}{1+v}}\right),
  \end{equation}
  where the ${O}(\cdot)$-notation omits $\log\log \frac{1}{\Delta_a}$ terms.

\end{theorem}
From Theorem \ref{thm:4} we can see that compared with the regret bound $O(\frac{\log^{2.5} T}{\epsilon}\sum_{\Delta_a>0}{(\frac{1}{\Delta_a})^{\frac{1}{v}}})$ in Theorem \ref{thm:2}, we achieve an improved bound of $O(\frac{\log T}{\epsilon}\sum_{\Delta_a>0}{(\frac{1}{\Delta_a})^{\frac{1}{v}}})$. { We think the main reason for the improvement is that the UCB-based method needs to make a reward-dependent choice in each round whereas the SE-based method only makes the reward-dependent choices in $K-1$ special rounds when it performs arm elimination.} Moreover, we also have an instance-independent regret bound. While in the bounded rewards case it has been shown that a DP variant of the SE algorithm is  optimal \citep{sajed2019optimal}, it is still unknown whether Algorithm \ref{Algo-DPSE} is optimal in the heavy-tailed case. In the following we study the lower bound of regret for  heavy-tailed MAB problem in the $\epsilon$-DP model.
We start from the two-armed instance-dependent regret lower bound which is specified in Theorem \ref{thm:6} to show that the dependency on the term of $\frac{1}{\epsilon}(\frac{1}{\Delta})^\frac{1}{v}$ is unavoidable in general. Due to space limitation, we put Theorem \ref{thm:6} and its proof in Appendix. Then we extend to the $K$-arm case and show the instance-dependent regret bound presented in Theorem \ref{thm:4} is {\bf optimal}.
The lower bound of the instance-independent regret is still unclear, and we leave it as an open problem.

\begin{theorem}[DP Instance-dependent Lower Bound]
\label{thm:7}
There exists a heavy-tailed $K$-armed bandit instance with $u\le1$ in (\ref{eq:2}), $\mu_a\le\frac{1}{6}$ and $\Delta_a\in(0,\frac{1}{12})$, such that for any $\epsilon$-DP ($0<\epsilon\le 1$) algorithm $\mathcal{A}$ whose expected regret is at most $T^{\frac{3}{4}}$, we have
    \begin{equation}
        \mathcal{R}_T\ge\Omega \left(\frac{\log T}{\epsilon}\sum_{\Delta_a>0}{\big(\frac{1}{\Delta_a}\big)^{\frac{1}{v}}}\right).
    \end{equation}
\end{theorem}
Below we will sketch the proof of Theorem \ref{thm:7}. Notably, previous hard instances in the bounded rewards case cannot provide tight lower bounds in our setting. Moreover, the hard instance in the non-private heavy-tailed MAB case \citep{bubeck2013bandits} is only for two-armed setting and cannot be generalized to $K$ arms. Therefore, it is necessary to construct new hard instance to prove Theorem \ref{thm:7}. Specifically, we design the following heavy-tailed MAB problem as a hard instance: the instance $\bar{P}$ where the distribution of each arm $a\in[K]$ is $\nu_a=\Big(1-\frac{s_a^{1+v}}{2}\Big)\delta_0+\frac{s_a^{1+v}}{2}\delta_{1/s_a},$
    with some $\frac{1}{2}\ge\mu_1\ge\cdots\ge\mu_K$ and $s_a=(2\mu_a)^{\frac{1}{v}}$. That is when $X\sim \nu_a$, $X=0$ with probability  $(1-\frac{s_a^{1+v}}{2})$, and $X=\frac{1}{s_a}$ with probability $\frac{s_a^{1+v}}{2}$. It is easy to verify for each $a\in[K]$ that $\mathbb{E}[\nu_a]=\mu_a$ and the $(1+v)$-th moment is bounded by 1.
    Now, we fix an arm $a\ne1$ and denote $\mathcal{E}_a$ as the event that the arm $a$ is pulled at most $t_a\triangleq\frac{\log T}{100\epsilon 4^{1/v}}(\frac{1}{\Delta_a})^\frac{1+v}{v}$ times. We will mainly show  that $\mathbb{P}_{\mathcal{A},\bar{P}}(\mathcal{E}_a)\le\frac{1}{2K}$. To prove this, we consider another instance $\bar{Q}_a$ where the distribution of any arm $a^\prime\ne a$ remains unchanged, and the distribution of arm $a$ is $\nu_a^\prime=[ 1-(\frac{s_a^{1+v}}{2}+2\Delta_a\gamma_a)]\delta_0+(\frac{s_a^{1+v}}{2} )\delta_{1/s_a}+(2\Delta_a\gamma_a)\delta_{1/\gamma_a},$
    where $\gamma_a=(4\Delta_a)^{\frac{1}{v}}$. Note that when setting $\mu_a^{1+v}\le\frac{1}{6}$ and $\Delta_a^{1+v}\le\frac{1}{12}$, we have $\frac{s_a^{1+v}}{2}+2\Delta_a\gamma_a=2^{\frac{1}{v}}\mu_a^{\frac{1+v}{v}}+2\cdot4^{\frac{1}{v}}\cdot\Delta_a^{\frac{1+v}{v}}\le(\frac{1}{3})^{\frac{1}{v}}+2\cdot(\frac{1}{3})^{\frac{1}{v}}\le1,$
    thus the postulated $\nu_a^\prime$ is reasonable. It is also easy to verify that $\mathbb{E}[\nu_a^\prime]=\mu_a+2\Delta_a=\mu_1+\Delta_a \leq 1$ and the $(1+v)$-th moment  of $v_a^\prime$ is bounded by 1. With the coupling lemma in \citet{KarwaV_ITCS18}, we can prove that $\mathbb{P}_{\mathcal{A},\bar{P}}(E_a)\le\frac{1}{2K}$. Thus, by taking the union bound we can get with probability at least $\frac{1}{2}$, the regret of $\bar{P}$ for any private algorithm $\mathcal{A}$ is $\Omega(\frac{\log T}{\epsilon}\sum_{\Delta_a>0}{(\frac{1}{\Delta_a})^{\frac{1}{v}}})$.
\begin{remark}\label{remark2}
In the MAB with bounded rewards case, it has been shown that the optimal rate of the expected rate is $O(\frac{K\log T}{\epsilon}+\sum_{\Delta_a>0} \frac{\log T}{\Delta_a})$ \citep{sajed2019optimal}. Compared with the optimal rate $O(\frac{\log T}{\epsilon}\sum_{\Delta_a>0}{(\frac{1}{\Delta_a})^{\frac{1}{v}}})$ in the heavy-tailed case, we can see there is a huge difference. First, the dependency on $\frac{1}{\Delta_a}$ now becomes to $(\frac{1}{\Delta_a})^{\frac{1}{v}}$. Secondly, the price of privacy in the bounded rewards case is an additional term of $O(\frac{K \log T}{\epsilon})$ compared with the non-private rate, while in the heavy-tailed case, there is an additional factor of $\frac{1}{\epsilon}$ compared with the non-private one.
\end{remark}

\section{LDP HEAVY-TAILED MAB}
\label{sec:LDP}
In this section, we will investigate upper and lower bounds for heavy-tailed MAB in the \emph{local} DP model. We start from the upper bounds. To design an $\epsilon$-LDP algorithm, most of the previous works on MAB with bounded rewards modifies the UCB algorithm and uses the Laplacian mechanism to guarantee the LDP property, such as \citep{chen2020locally,zhou2020local}. For sub-Gaussian rewards, \citet{ren2020multi} first map the unbounded rewards with a Sigmoid function and then use the Laplacian mechanism and UCB algorithm. Moreover, it has been shown that this type of method is near-optimal \citep{chen2020locally,ren2020multi}. Thus, for MAB with heavy-tailed rewards, a straightforward way is to modify the UCB-based method. Specifically, each reward will be shrunken to a certain range and then added Laplacian noise, subsequently the algorithm uses the confidence bound on these perturbed rewards to pull an arm. However, such an approach may cause enormous error. The reason is that, similar to Algorithm \ref{Algo-UCB}, here the threshold of each reward depends on $n_a$. That is, the Laplacian noise we added for each reward will be proportional to $n_a$. As a result, unlike the Tree-based mechanism in the central model, this LDP version of the UCB algorithm will introduce a huge amount of error to estimate the mean.

To achieve a better utility, we propose an $\epsilon$-LDP version of the SE algorithm, see Algorithm \ref{Algo-LDPSE} for details. The basic idea is similar to Algorithm \ref{Algo-DPSE}, where the algorithm now maintains (private) confidence interval for each arm via the perturbed rewards instead of the noisy average. However, compared with the above LDP version of the UCB algorithm, we can see that here the Laplacian noise added to each reward is \emph{independent} on the number of rounds  $n_a$, which could be much smaller than the noise added in the LDP version of UCB method when $T$ is sufficiently large. The following two theorems provide the privacy and utility guarantees for Algorithm~\ref{Algo-LDPSE}, respectively.


\begin{theorem}\label{DPofLDP}
For any $\epsilon>0$, Algorithm \ref{Algo-LDPSE} is $\epsilon$-local differentially private.
\end{theorem}

\begin{theorem}[LDP Upper Bound]
\label{UpperBoundLDP}
Set $\beta=\frac{1}{T}$ in Algorithm~\ref{Algo-LDPSE}. For any $\epsilon \in (0,1]$ and sufficiently large $T$, the instance-dependent expected regret of Algorithm~\ref{Algo-LDPSE} satisfies
    \begin{equation}
          \mathcal{R}_T \leq {O} \left(\frac{ u^{\frac{2}{v}}\log T}{\epsilon^2}\sum_{\Delta_a>0}{\Big(\frac{1}{\Delta_a}\Big)^{\frac{1}{v}}}+\max_a\Delta_a\right).
    \end{equation}
Moreover, the instance-independent expected regret of Algorithm~\ref{Algo-LDPSE} satisfies
  \begin{equation}
    \label{eq:instance-independent-local-upper}
     \mathcal{R}_T \leq  O\left(u^{\frac{2}{1+v}}\left(\frac{K\log T}{\epsilon^2}\right)^{\frac{v}{1+v}}T^{\frac{1}{1+v}}\right),
  \end{equation}
where {the ${O}(\cdot)$-notations omit $\log \log \frac{1}{\Delta_a}$ terms.}

\end{theorem}

\begin{algorithm}[!t]
    \caption{LDP Robust Successive Elimination}
    \label{Algo-LDPSE}
    \begin{algorithmic}[1]
        \Require Confidence $\beta$, parameters $\epsilon,v,u$.
        \State $\mathcal{S}\gets\{1,\cdots,K\}$
        \State Initialize: $t\gets0$, $\tau\gets0$.
        \Repeat
            \State $\tau\gets \tau+1$.
            \State Set $\bar{\mu}_a=0$ for all $a\in\mathcal{S}$.
            \State $r\gets 0$, $D_\tau\gets4^{-\tau}$.
            \State $R_\tau\gets  \lc u^{\frac{2}{v}}
            (
            \frac
            {
                28^{2(1+v)/v}
                \log
                    (
                        8|\mathcal{S} |\tau^2/\beta
                    )
            }
            {
                \epsilon^2 D_\tau^{2(1+v)/v}
            }
            )+\log
                    (
                        \frac{8|\mathcal{S} |\tau^2}{\beta}
                    ) \rc$.
            \State $B_\tau\gets
                    (
                    \frac
                    {
                        u\sqrt{R_\tau}\epsilon
                    }
                    {
                       \sqrt{ \log( 8|\mathcal{S}|\tau^2/\beta)}
                    }
                    )^{1/(1+v)}$.
            \While {$r<R_\tau$}
                \State $r\gets r+1$.
                \For {$a\in\mathcal{S}$}
                    \State $t\gets t+1$.
                    \State Sample a reward $x_{a,r}$ for each arm $a\in\mathcal{S}$.
                \State $\widetilde{x}_{a,r}\gets x_{a,r}\cdot\mathbb{I}_{\{|x_{a,r}|\le B_\tau\}}$.
                \State $\widehat{x}_{a,r}\gets\widetilde{x}_{a,r}+{\rm Lap}(\frac{2B_\tau}{\epsilon})$
                \EndFor
            \EndWhile
             \State For each $a\in \mathcal{S}$, compute $\bar{\mu}_a\gets(\sum\limits_{l=1}^{R_\tau}\widehat{x}_{a,l})/R_\tau$.
             \State $\bar{\mu}_{\rm max}\gets\max_{a\in\mathcal{S}}\bar{\mu}_i$.
             \State $err_\tau\gets u^{1/(1+v)}
            (\frac
            {\sqrt{\log(8|\mathcal{S}|\tau^2/\beta)}}
            {R_\tau\epsilon})^{v/(1+v)}$.
            \For {all viable arm $a$}
                \If {$\widetilde{\mu}_{\rm max}-\widetilde{\mu}_a>14err_\tau$} \State Remove arm $a$ from $\mathcal{S}$. \EndIf
            \EndFor
        \Until {$|\mathcal{S}|=1$}
        \State Pull the arm in $\mathcal{S}$ in all remaining $T-t$ rounds.
    \end{algorithmic}
\end{algorithm} 
In the following, we derive both instance-dependent and instance-independent lower bounds for heavy-tailed MAB in the $\epsilon$-LDP model. Similar to instance-dependent lower bounds in central DP, we also first analyze the two-armed case to show that the dependency on the term of $\frac{1}{\epsilon^2}(\frac{1}{\Delta})^\frac{1}{v}$ is unavoidable in general (see Theorem \ref{LDPInstDepend} in Appendix~\ref{sec:appendix-LDP}) and then extend to $K$-armed case in the following result.

\begin{theorem}[LDP Instance-dependent Lower Bound]
\label{LDPKIDLB}
There exists a heavy-tailed K-armed  bandit instance with $u\le 1$ in (\ref{eq:2}) and $\Delta_a\triangleq\mu_1-\mu_a \in (0,\frac{1}{5})$, such that for any $\epsilon$-LDP ($0<\epsilon\le 1$) algorithm whose regret $\leq o(T^\alpha)$ for any $\alpha > 0$, the regret satisfies
\[
\liminf _{T \rightarrow \infty} \frac{\mathcal{R}_T}{\log T}\geq  \Omega \left(\frac{1}{\epsilon^2}\sum_{\Delta_a>0}(\frac{1}{\Delta_a})^{\frac{1}{v}} \right).
\]
\end{theorem}

\begin{remark}
Theorem \ref{LDPInstDepend} reveals that the term of $\frac{1}{\epsilon^2 \Delta^{\frac{1}{v}}}$ is unavoidable in the regret bound. As a result, the attained bound in Theorem \ref{UpperBoundLDP} is  optimal. Compared with the optimal rate $O(\frac{1}{\Delta^{\frac{1}{v}}})$ in the non-private case, we can see the price of privacy is an additional factor of $\frac{1}{\epsilon^2}$, which is similar to other MAB with bounded/sub-Gaussian rewards problems  in the LDP model~\citep{zhou2020local,ren2020multi}.
\end{remark}
\begin{theorem}[LDP Instance-independent Lower Bound]
\label{LDPIndepLower}
There exists a heavy-tailed $K$-armed bandit instance with the $(1+v)$-th bounded moment of each reward distribution is bounded by $1$. Moreover, if $T$ is large enough, for any the $\epsilon$-LDP algorithm $\mathcal{A}$ with $\epsilon \in (0,1]$, the expected regret must satisfy
$$\mathcal{R}_T \geq \Omega\left( \Big(\frac{K}{\epsilon^2}\Big)^{\frac{v}{1+v}} T^{\frac{1}{1+v}}\right).$$
\end{theorem}
\begin{remark}
From Theorem \ref{LDPIndepLower}, we can see the upper bound~\eqref{eq:instance-independent-local-upper} of Algorithm \ref{Algo-LDPSE} is nearly optimal. However, compared with instance-independent lower bound, there is still a $\text{poly}(\log T)$ factor gap. We conjecture this factor could be removed by using some more advanced robust estimator, such as the estimator in \citet{lee2020optimal} and we will leave it as an open problem. For MAB with bounded rewards in the LDP model, \citet{basu2019differential} shows that its instance-dependent regret bound is always at least $\Omega(\frac{\sqrt{KT}}{\epsilon})$, {\em i.e.,} there is an additional factor of $\frac{1}{\epsilon}$ compared with the non-private case. However, for heavy-tailed MAB, compared with the lower bound of $\Omega(K^{\frac{v}{1+v}}T^{\frac{1}{1+v}})$ in the non-private case, from Theorem \ref{LDPIndepLower} we can observe that the difference is  a factor of $(\frac{1}{\epsilon^2})^{\frac{v}{1+v}}$. Thus, combining with Remark \ref{remark2}, we can conclude that heavy-tailed MAB and bounded MAB are quite different in both central and local differential privacy models.
\end{remark}
\section{EXPERIMENTS}
\label{sec:experiments}
In this section, we conduct experiments on synthetic datasets to evaluate the performance of our algorithms. Since this is the first paper studying DP/LDP heavy-tailed MAB and there is no previous methods, we will only evaluate the performance of our algorithms. Our experiments consist of two parts. In the first part, we empirically compare our Algorithm~\ref{Algo-UCB} and Algorithm~\ref{Algo-DPSE} in the central DP model. In the second part, we will evaluate our Algorithm~\ref{Algo-LDPSE} in the LDP model.
\paragraph{Datasets and Setting}
For the data generation, we follow similar settings as in the previous work on MAB with heavy-tailed rewards, such as \citet{lee2020optimal}. Specifically, we set $K=5$ and restrict the mean of each arm within $[0.1, 0.9]$ throughout the experiment. We consider three instances, denoted by $S_1$, $S_2$ and $S_3$. In $S_1$, we let the gaps of sub-optimal arms decrease linearly where the largest mean is always $0.9$ and the smallest mean is always $0.1$ (so the means are \{0.9, 0.7, 0.5, 0.3, 0.1\}). In $S_2$, we consider the case that a larger fraction of arms have large sub-optimal gaps, hence we set the mean of each arm $a$ by a quadratic convex function $\mu_a=0.05(a-5)^2+0.1$ (so the means are \{0.9, 0.55, 0.3, 0.15, 0.1\}). In $S_3$, we consider the case that a larger fraction of arms have small sub-optimal gaps, hence we set the mean of each arm $a$ by a quadratic concave function $\mu_a=-0.05(a-1)^2+0.9$ (so the means are \{0.9, 0.85, 0.7, 0.45, 0.1\}).
In all the settings, the reward of each arm $a\in[K]$ at each pull is drawn from a Pareto distribution with shape parameter $\alpha$ and scale parameter $\lambda_a$. Specifically, each time after pulling arm $a$, the learner receives a reward $x$ which follows the following probability density function:
\begin{equation*}
    f(x)=
    \begin{cases}
        \frac{\alpha\lambda_a^\alpha}{x^{\alpha+1}} &x\ge\lambda_a\\
        0 &x<\lambda_a.
    \end{cases}
\end{equation*}
We adopt the Pareto distribution because  it is common in practice. We set $\alpha=1.05+v$ such that the $1+v$-th moment of reward distribution always exists and is bounded by $\frac{\alpha\lambda_a^{1+v}}{\alpha-(1+v)}$. For a given $\alpha$, we set $\lambda_a=\frac{(\alpha-1)\mu_a}{\alpha}$ since the mean of a Pareto distribution with parameters $\alpha$ and $\lambda_a$ is $\frac{\alpha\lambda_a}{\alpha-1}$. We take the maximum of $\frac{\alpha\lambda_a^{1+v}}{\alpha-(1+v)}$ among all arms $a\in[K]$ as $u$ in the experiments.

For each algorithm we run $90$ independent repetitions for each case. In figures, we show the average of cumulative regret (represented by the solid line) for comparing the performance of algorithms and error bars of a quarter standard deviation (represented by the shaded region) for comparing the robustness of algorithms.


\begin{figure}[!t]
    \centering
    \subcaptionbox{$v=0.5, \epsilon=0.5$}{\includegraphics[width=.493\linewidth]{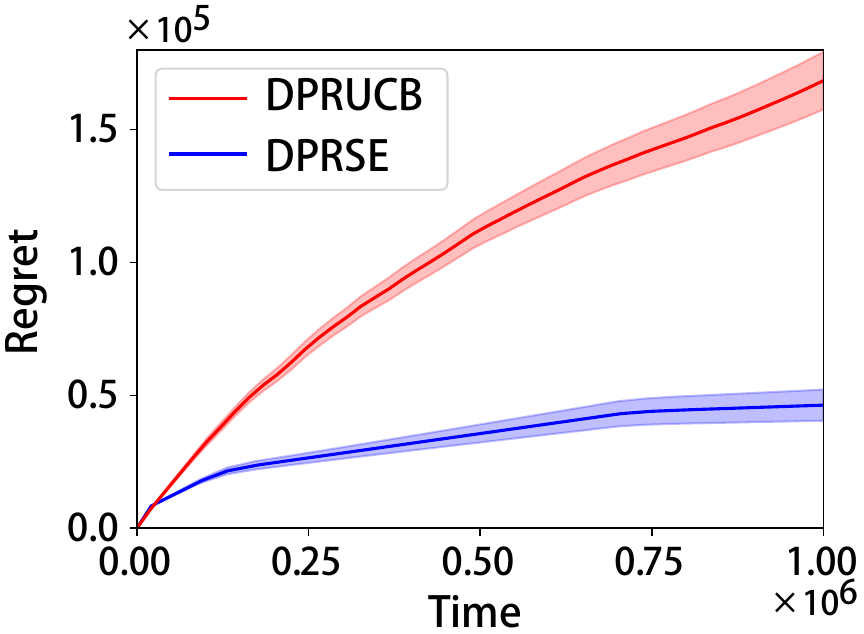}}
    \subcaptionbox{$v=0.5, \epsilon=1.0$}{\includegraphics[width=.493\linewidth]{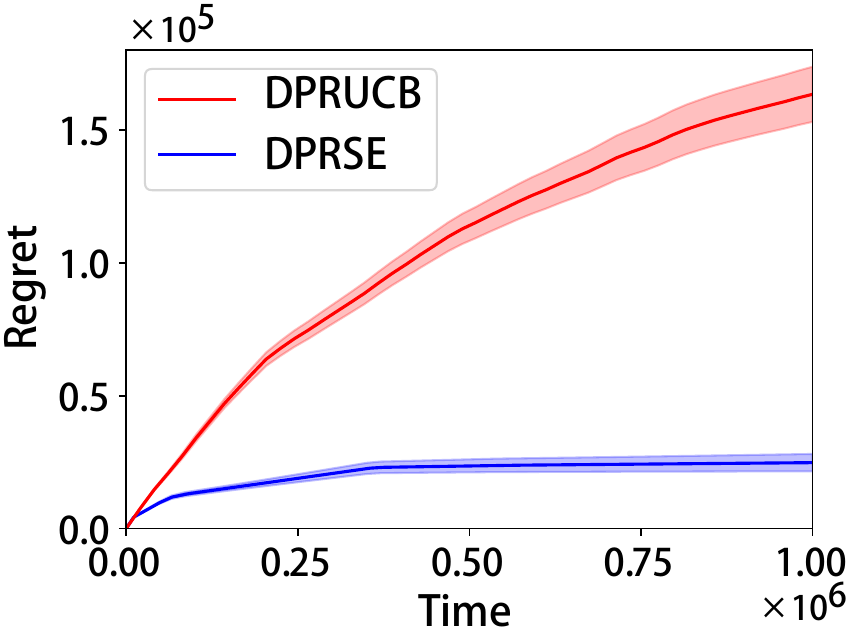}}
    \subcaptionbox{$v=0.9, \epsilon=0.5$}{\includegraphics[width=.493\linewidth]{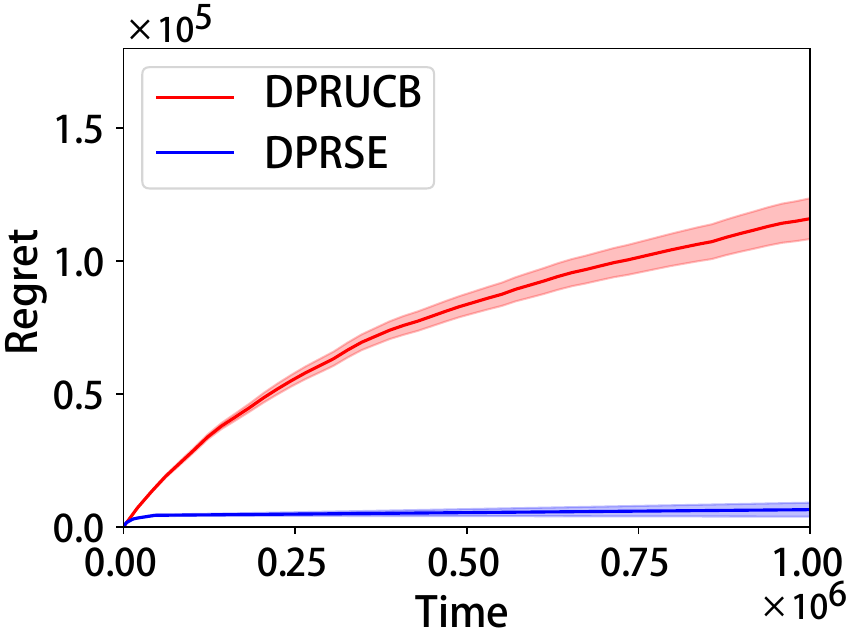}}
    \subcaptionbox{$v=0.9, \epsilon=1.0$}{\includegraphics[width=.493\linewidth]{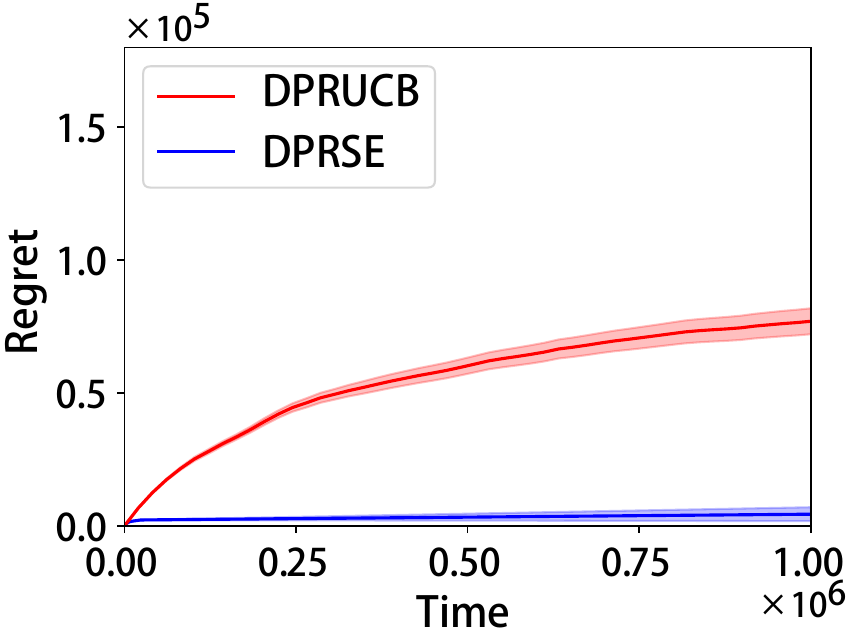}}
    \caption{DP Setting 1 ($S_1$)}
    \label{fig:DP1}
\end{figure}

\paragraph{Results and Discussion.} In the first part, we compare the performance of our proposed Algorithm~\ref{Algo-UCB} (DPRUCB) and Algorithm~\ref{Algo-DPSE} (DPRSE) for the central DP model. In each setting, we vary $\epsilon\in\{0.5, 1.0\}$ and $v\in\{0.5, 0.9\}$. The results of setting $S_1$ is given in Figure~\ref{fig:DP1}. And the results of setting $S_2$ and $S_3$ are presents in Figure \ref{fig:DP2} and Figure \ref{fig:DP3}, which are included in  Appendix~\ref{sec:appendix-experiments} due to space limitation. From these results, we can see that DPRSE always outperform DPRUCB among these three settings. Moreover, we can see when $\epsilon$ decreases, the regret will increase, and when the term $v$ becomes larger, we have smaller regret.

\begin{figure}[!htb]
    \centering
    \subcaptionbox{$v=0.5$}{\includegraphics[width=.493\linewidth]{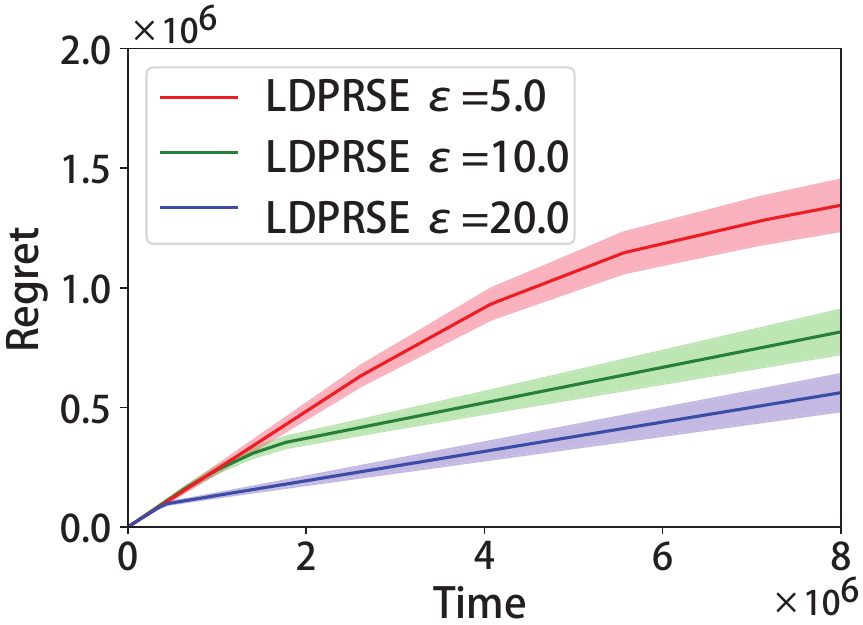}}
    \subcaptionbox{$v=0.9$}{\includegraphics[width=.493\linewidth]{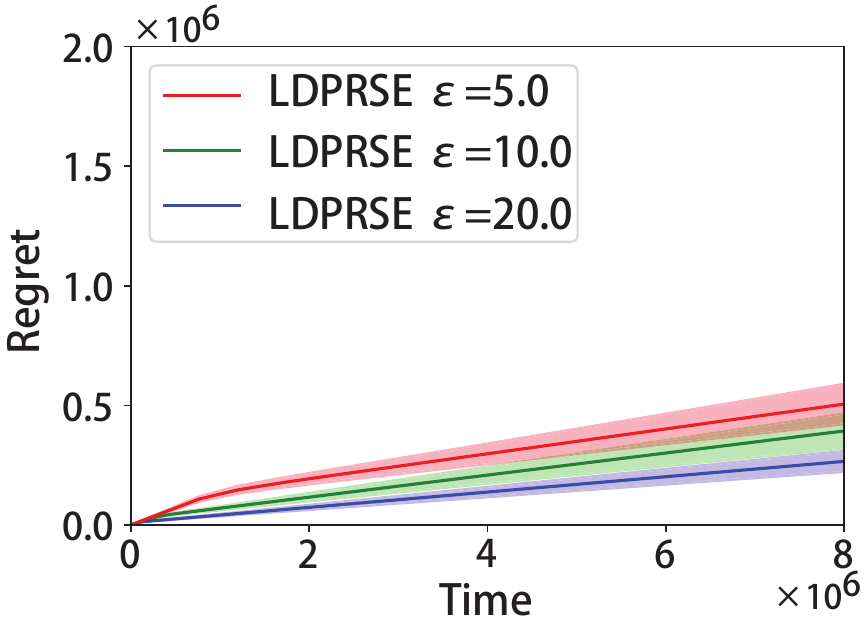}}
    \caption{LDP Setting 3 ($S_3$)}
    \label{fig:LDP}\vspace{-2mm}
\end{figure}

In the second part, we evaluate our proposed Algorithm~\ref{Algo-LDPSE} (LDPRSE) for the local DP model in the Setting 3 ($S_3$), since $S_3$  will has larger regret than $S_1$ and $S_2$ theoretically. We vary $\epsilon\in\{5.0, 10.0, 20.0\}$ and $v\in\{0.5, 0.9\}$, The results are given in Figure \ref{fig:LDP}. We can see that when $\epsilon$ or $v$ is larger, we have smaller regret. Moreover, compared with the central DP model, we can see the regret is larger in the LDP model.

In summary, we can observe that all the above results support our previous theoretical analysis.

\section{CONCLUSIONS}

In this paper, we provided the first study on the problem of MAB with heavy-tailed reward distributions in the (local) Differential Privacy model. We mainly focused on the case the reward distribution of each arm only has $(1+v)$-th moment with some $v\in (0, 1]$. In the central $\epsilon$-DP model, we first provided a near optimal result by developing a private and robust UCB algorithm. To achieve this we provided an adaptive version of the Tree-based mechanism. Then, we improved the result via a private and robust version of the SE algorithm. Finally, we showed that the instance-dependent regret bound of our improved algorithm is optimal by showing its lower bound. In the $\epsilon$-LDP model. We proposed an algorithm which could be seen a locally private and robust version of the SE algorithm, which provably achieve  (near) optimal rates for both  instance-dependent and instance-independent regret.

There are still many open problems besides the future work mentioned in the main context. First, throughout the whole paper we need to assume both $u$ and $v$ are known. How to address a more practical case where they are unknown? Recent work that addresses this issue for standard bandit problems~\citep{ashutosh2021bandit}, while it remains unknown whether they can be extended to the private case.
Secondly, for the setting of MAB with bounded reward, it has been shown that an UCB-based private algorithm can also attain an optimal regret guarantee. Thus, a natural question is whether it is possible to get an optimal DP variant of UCB algorithm for our problem.

\subsubsection*{Acknowledgements}
Di Wang and Yulian Wu were support in part by the baseline funding BAS/1/1689-01-01 and funding from the AI Initiative REI/1/4811-10-01 of King Abdullah University of Science and Technology (KAUST). Peng Zhao was supported by the Baidu Scholarship.

\bibliography{aistats}


\clearpage
\appendix

\thispagestyle{empty}

\onecolumn \makesupplementtitle

\section{TECHNICAL LEMMAS}\label{sec:appendix-tech}
\begin{lemma}[Tail Bound of Laplacian Vairable \citep{dwork2006calibrating}]\label{alemma2}
    If $X\sim{{\rm Lap}(b)}$, then
    \begin{equation*}
        \mathbb{P}(|X|\ge t\cdot b)=\exp(-t).
    \end{equation*}
\end{lemma}
\begin{lemma}[Bernstein's Inequality \citep{vershynin2018high}]\label{alemma3}
Let $X_1, \cdots X_n$ be $n$ independent zero-mean random variables. Suppose $|X_i|\leq M$ and $\mathbb{E}[X_i^2]\leq s$ for all $i \in [n]$. Then for any $t>0$, we have
\begin{equation*}
    \mathbb{P}\left\{\frac{1}{n}\sum_{i=1}^n X_i \geq t \right\}\leq \exp\left(-\frac{\frac{1}{2}t^2n}{s+\frac{1}{3}Mt}\right)
\end{equation*}
\end{lemma}
\begin{lemma}\label{alemma4}
Given a random variable $X$ with $\mathbb{E}[|X|^{1+v}]\leq u$ for some $v\in (0, 1]$, for any $B>0$ we have
\begin{equation*}
    \mathbb{E}\left[X \cdot \mathbb{I}_{|X|>B}\right]\leq \frac{u}{B^v}.
\end{equation*}
\end{lemma}
\begin{proof}[{of Lemma~\ref{alemma4}}]
By the definition of expectation, we have
\begin{align*}
   u\geq  \mathbb{E}[|X|^{1+v}] &= \int_{0}^{\infty} (1+v) t^{1+v-1} \mathbb{P}(|X|> t) \diff{t} \\
   &\geq \int_{B}^{\infty} t^{v} \mathbb{P}(|X|> t) \diff{t} \\
   &\geq B^v \int_{B}^{\infty}\mathbb{P}(|X|> t)\diff{t} \\
   &= B^v \int_{0}^{\infty}\mathbb{P}(X \cdot \mathbb{I}_{|X|> B}>t)\diff{t}\\
   &= B^v      \mathbb{E}\left[X \cdot \mathbb{I}_{|X|>B}\right].
\end{align*}
Rearranging the inequality finishes the proof.
\end{proof}

\begin{lemma}[Concentration of Laplace Variables~\citep{wang2018empirical}]\label{alemma5}
If $X_1, \cdots X_n \sim \operatorname{Lap}(s/\epsilon)$, then with probability at least $1-\beta$, we have
\[
	\left|\frac{1}{n}\sum_{i=1}^{n}X_i\right|\leq \frac{2s}{\epsilon\sqrt{n}}\sqrt{\log \frac{2}{\beta}}.
\]
\end{lemma}

\begin{lemma}[Jensen’s Inequality] \label{LemmaJensen}
Let $X$ be an integrable, real-valued random variable, and $\psi$ be a convex function. Then $\psi(\mathbb{E}[X]) \leq \mathbb{E}[\psi(X)]$.
\end{lemma}

\begin{lemma}[Relation between Raw Moment and Central Moment]\label{lemma-raw-central}

Let $X$ be a random variable over $\mathbb{R}$ such that $\mathbb{E}[X]=\mu$. We have the following two results:
\begin{enumerate}
    \item[(1)] When $\mathbb{E}[|X|^{1+v}]<\infty$ for some $v\in(0,1]$, we have
    \[
        \mathbb{E}[|X-\mu|^{1+v}]\le4\mathbb{E}[|X|^{1+v}]<\infty.
    \]
    \item[(2)] When $\mathbb{E}[|X-\mu|^{1+v}]<\infty$ for some $v\in(0,1]$, we have
\begin{equation*}
      \mathbb{E}[|X|^{1+v}]\leq 2\mathbb{E}[|X-\mu|^{1+v}]+ 2|\mu|^{1+v}<\infty.
\end{equation*}
\end{enumerate}
\end{lemma}

\begin{proof}[{of Lemma~\ref{lemma-raw-central}}]
    When $ \mathbb{E}[|X|^{1+v}]<\infty$, we have
    \begin{flalign*}
        \mathbb{E}[|X-\mu|^{1+v}]
        &\le\mathbb{E}[|X-\mu|^{1+v}+|X+\mu|^{1+v}]\\
        &\le\mathbb{E}[2(|X|^{1+v}+|\mu|^{1+v})]\\
        &=2\mathbb{E}[|X|^{1+v}]+2\mathbb{E}[|\mu|^{1+v}]\\
        &=2\mathbb{E}[|X|^{1+v}]+2|\mu|^{1+v}\\
        &=2\mathbb{E}[|X|^{1+v}]+2|\mathbb{E}[X]|^{1+v}\\
        &\le2\mathbb{E}[|X|^{1+v}]+2\mathbb{E}[|X|^{1+v}]\\
        &=4\mathbb{E}[|X|^{1+v}],
    \end{flalign*}
    where the second inequality is due to the inequality $(6)$ of \citet{Bahr1965inequalities}, and the last inequality is due to Jensen’s inequality (Lemma~\ref{LemmaJensen}).

    When $ \mathbb{E}[|X-\mu|^{1+v}]<\infty$, we have
    \begin{flalign*}
        \mathbb{E}[|X|^{1+v}]
        &\le \mathbb{E}[|X|^{1+v}]+\mathbb{E}[|X-2\mu|^{1+v}]\\
        &\le \mathbb{E}[2(|X-\mu|^{1+v}+|\mu|^{1+v})]\\
        &= 2\mathbb{E}[(|X-\mu|^{1+v})]+2|\mu|^{1+v},
    \end{flalign*}
    where the second inequality comes from the inequality $(6)$ of \citet{Bahr1965inequalities} as well.
\end{proof}


\section{OMITTED PROOFS FOR SECTION~\ref{sec:central} (central differential privacy)}

\begin{proof}[{of Lemma \ref{lemma-tree}}]
The proof can be directly followed by the proof in the original bounded case, which is given by \citep{chan2011private,dwork2010differential}.
\end{proof}

\begin{proof}[{of Theorem \ref{theorem-privacy-ucb}}]
    Note that all rewards generated by arm $a$ are inserted into the corresponding Tree-based Mechanism $Tree_a$ and each Tree-based Mechanism is $\epsilon$-differentially private (by Lemma \ref{lemma-tree}), the proof follows directly from the Parallel Composition Theorem of DP.
\end{proof}

\begin{proof}[{of Lemma \ref{lemma-esterr}}]
We will prove the following more general lemma:
\begin{lemma}\label{lemma:a1}
    In Algorithm \ref{Algo-UCB}, for a fixed arm $a$ and $t$, we have the following estimation error with probability at least $1-2\delta$ for any $\delta \geq \frac{1}{2T^4}$
    \begin{equation}
        \widehat{\mu}_{a}(n_a, t)\le \mu_a+18u^{\frac{1}{1+v}}
        \left(\frac{\log \frac{1}{\delta} \log^{1.5+\frac{1}{v}}T}
        {n_a\epsilon}
        \right)^{\frac{v}{1+v}}.
    \end{equation}

\end{lemma}
  Thus,  in Algorithm~\ref{Algo-UCB}, we set $\delta$ to be $\frac{1}{2t^4}$. Hence with probability at least $1-\frac{1}{t^4}$, we have
     \begin{equation*}
         \widehat{\mu}_{a}(n_a, t)\le \mu_a+18u^{\frac{1}{1+v}}
        \left(\frac{\log(2t^4)\log^{1.5+\frac{1}{v}}T}
        {n_a\epsilon}
        \right)^{\frac{v}{1+v}}.
     \end{equation*}
     In the following we will prove Lemma \ref{lemma:a1}.

    Denote the total noise introduced by the $\tree_{a}$ by $noise$ and the $i$-th reward obtained from arm $a$ by $x_{a,i}$. From Bernstein's inequality for bounded random variables, noting that $\mathbb{E}(X_a^2 \cdot \mathbb{I}_{|X_a|\le B})\le uB^{1-v}$ if $\mathbb{E}|X|^{1+v} \leq u$, we have, with probability at least $1-2\delta$,
    \begin{align*}
        &
        \left| \widehat{\mu}_a(n_a, t)-\mu_a\right|\\
        &
        =\left| \frac{1}{n_{a}}\sum_{i=1}^{n_a}x_{a,i}\mathbb{I}_{\left| x_{a,i}\right|\le B_{i}}+noise-\mu_a\right| \\
        &\le \left| \frac{1}{n_{a}}\sum_{i=1}^{n_a}x_{a,i}\mathbb{I}_{\left| x_{a,i}\right|\le B_{i}}-\mu_a\right|+\left|noise\right|\\
        &
        =\left| \frac{1}{n_{a}}\sum_{i=1}^{n_a}\left[x_{a,i}\mathbb{I}_{\left| x_{a,i}\right|\le B_{i}}-\mathbb{E}(X_a\mathbb{I}_{\left| X_a\right|\le B_{i}})\right]+\frac{1}{n_{a}}\sum_{i=1}^{n_a}\left[\mathbb{E}(X_a\mathbb{I}_{\left| X_a\right|\le B_{i}})-\mathbb{E}X_a\right]\right|+\left|noise\right|\\
        &
        \le \left|\frac{1}{n_{a}}\sum_{i=1}^{n_a}\left[x_{a,i}\mathbb{I}_{\left| x_{a,i}\right|\le B_{i}}-\mathbb{E}(X_a\mathbb{I}_{\left| X_a\right|\le B_{i}})\right] \right|+\left|\frac{1}{n_{a}}\sum_{i=1}^{n_a}\mathbb{E}(X_a\mathbb{I}_{\left| X_a\right|> B_{i}})\right|+|noise|\\
        &\le
      \sqrt{\frac{2B_{n_a}^{1-v}u\log(\frac{1}{\delta})}{n_a}}+\frac{B_{n_a}\log\frac{1}{\delta}}{3n_a}+  \frac{1}{n_a}\sum_{i=1}^{n_a}\frac{u}{B_{i}^v}+\frac{2B_{n_a}\log\frac{1}{\delta}\log^{1.5}T}{\epsilon n_a}.
    \end{align*}
    Where the last inequality is due to     Lemma \ref{alemma4} and Lemma \ref{alemma2}, and Lemma \ref{alemma3} that with probability at least $1-\delta$,
    \begin{equation*}
        \left|\frac{1}{n_{a}}\sum_{i=1}^{n_a}\left[x_{a,i}\mathbb{I}_{\left| x_{a,i}\right|\le B_{i}}-\mathbb{E}(X_a\mathbb{I}_{\left| X_a\right|\le B_{i}})\right] \right|\leq \sqrt{\frac{2B_{n_a}^{1-v}u\log(\frac{1}{\delta})}{n_a}}+\frac{B_{n_a}\log\frac{1}{\delta}}{3n_a}.
    \end{equation*}

    Recall that $B_{n}=\left(\frac{\epsilon un}{\log^{1.5}T}\right)^{\frac{1}{1+v}}$ for any $n\in\mathbb{N}^+$, with some evident calculations, we can bound each term above in the last inequality respectively. 
    \begin{align}
        \frac{1}{n_a}\sum_{i=1}^{n_a}\frac{u}{B_{i}^v}
        &=\frac{u}{n_a}\sum_{i=1}^{n_a}\left(\frac{\log^{1.5}T}{\epsilon ui}\right)^\frac{v}{1+v} \le \frac{u^{\frac{1}{1+v}}}{n_a}\frac{(\log^{1.5}T)^{\frac{v}{1+v}}}{\epsilon^{\frac{v}{1+v}}}\sum_{i=1}^{n_a}i^{-\frac{v}{1+v}} \notag \\
        &\le \frac{u^{\frac{1}{1+v}}}{n_a}\frac{(\log^{1.5}T)^{\frac{v}{1+v}}}{\epsilon^{\frac{v}{1+v}}}\cdot (1+v)\cdot n_a^{\frac{1}{1+v}}
        \le 2u^{\frac{1}{1+v}}\left(\frac{\log\frac{1}{\delta}\log^{1.5}T}{\epsilon n_a}\right)^\frac{v}{1+v}, \label{aeq:14}\\
        \sqrt{\frac{2B_{n_a}^{1-v}u\log(\frac{1}{\delta})}{n_a}}
        &
        =\frac{\sqrt{2}u^{\frac{1}{1+v}}\epsilon^{\frac{1-v}{2(1+v)}}(\log\frac{1}{\delta})^{\frac{1}{2}}}{n_a^{\frac{v}{1+v}}(\log^{1.5}T)^{\frac{1-v}{2(1+v)}}}\le \sqrt{10} u^{\frac{1}{1+v}}\left(\frac{\log\frac{1}{\delta}\log^{1.5}T}{\epsilon n_a}\right)^\frac{v}{1+v}, \label{aeq:15}\\
        \frac{B_{n_a}\log\frac{1}{\delta}}{3n_a}
        &
        \le\frac{5B_{n_a}(\log\frac{1}{\delta})^{\frac{v}{1+v}}(\log^{1.5}T)^{\frac{1}{1+v}}}{3n_a}
        \le\frac{5(\epsilon un)^\frac{1}{1+v}(\log\frac{1}{\delta})^{\frac{v}{1+v}}(\log^{1.5}T)^{\frac{1}{1+v}}}{3(\log^{1.5}T)^{\frac{1}{1+v}}n_a}\notag\\
        &
        =\frac{5\epsilon^{\frac{1}{1+v}}u^{\frac{1}{1+v}}(\log\frac{1}{\delta})^{\frac{v}{1+v}}}{3n_a^{\frac{v}{1+v}}}
        \le 2u^{\frac{1}{1+v}}\left(\frac{\log\frac{1}{\delta}\log^{1.5}T}{\epsilon n_a}\right)^\frac{v}{1+v}. \label{aeq:16}
    \end{align}
   The inequality of the second term is based on the fact that $\epsilon\leq 1$ and
   \begin{align*}
       \left(\log \frac{1}{\delta}\right)^{\frac{1}{2}-\frac{v}{1+v} }=  \left(\log \frac{1}{\delta}\right)^\frac{1-v}{2(1+v)} \leq (\log 2T^4)^\frac{1-v}{2(1+v)}\leq (5\log T)^\frac{1-v}{2(1+v)}\leq \sqrt{5} (\log^{1.5}T)^{\frac{1-v}{2(1+v)}}.
   \end{align*}
    The inequality of the third term holds due to the following fact:
    \begin{align*}
        \frac{\log \frac{1}{\delta}}{(\log^{1.5}T)^\frac{1}{(1+v)}}\leq \left( \frac{\log \frac{1}{\delta}}{\log^{1.5}T} \right)^\frac{1}{(1+v)} \left(\log \frac{1}{\delta}\right)^\frac{v}{1+v}\leq 5 \left(\log \frac{1}{\delta}\right)^\frac{v}{1+v}.
    \end{align*}
    Moreover we have with probability at least $1-\delta$
    \begin{align}
        &\frac{2B_{n_a}\log\frac{1}{\delta}\log^{1.5}T}{\epsilon n_a}=\frac{2u^\frac{1}{1+v}(\log\frac{1}{\delta} \log^{1.5} T)^\frac{v}{1+v}}{(\epsilon n_a)^\frac{v}{1+v}}(\log\frac{1}{\delta})^{\frac{1}{1+v}} \notag\\
        &\leq \frac{10u^\frac{1}{1+v}(\log\frac{1}{\delta} \log^{1.5} T)^\frac{v}{1+v}}{(\epsilon n_a)^\frac{v}{1+v}}\log^\frac{v}{1+v} T\leq \frac{10u^\frac{1}{1+v}(\log\frac{1}{\delta} \log^{1.5} T)^\frac{v}{1+v}}{(\epsilon n_a)^\frac{v}{1+v}}\log^\frac{1}{1+v} T \notag\\
        &\leq  \frac{10u^\frac{1}{1+v}(\log\frac{1}{\delta} \log^{1.5+\frac{1}{v}} T)^\frac{v}{1+v}}{(\epsilon n_a)^\frac{v}{1+v}}. \label{aeq:17}
    \end{align}

    Based on (\ref{aeq:14}), (\ref{aeq:15}), (\ref{aeq:16}), (\ref{aeq:17}), we obtain
    \begin{equation*}
        \left| \widehat{\mu}_a(n_a, t)-\mu_a\right|\le18u^{\frac{1}{1+v}}\left(\frac{\log\frac{1}{\delta}\log^{1.5+\frac{1}{v}}T}{\epsilon n_a}\right)^\frac{v}{1+v},
    \end{equation*}
\end{proof}

\begin{proof}[{of Theorem \ref{thm:2}}]
    We denote by $N_a(t)$ the (random) number of times arm $a$ is selected up to time $t$ and let $I_{s,t}=18u^{\frac{1}{1+v}}\left(\frac{\log[2(t+1)^4]\log^{1.5+\frac{1}{v}}T}{\epsilon s}\right)^\frac{v}{1+v}$. We first derive the upper bound on $\mathbb{E}[N_a(T)]$ for any arm $a$. Let $\ell$ be an arbitrary positive integer.
    \begin{align*}
        N_a(T)
        &
        =1+\sum_{t=K+1}^T{\mathbb{I}_{a_t=a}}\\
        &
        \le\ell+\sum_{t=K+1}^T{\mathbb{I}_{a_t=a\text{ and }N_a(t-1)\ge\ell}}\\
        &
        \le\ell+\sum_{t=K+1}^T{\mathbb{I}_{\hat{\mu}_{a^*}(n_{a^*},t-1)+I_{N_{a^*}(t-1),t-1}\le\hat{\mu}_a(n_a,t-1)+I_{N_a(t-1),t-1}\text{ and }N_a(t-1)\ge\ell}}\\
        &
        \le\ell+\sum_{t=K+1}^T{\mathbb{I}_{\min\limits_{0<s<t}\hat{\mu}_{a^*}(s,t-1)+I_{s,t-1}\le\max\limits_{\ell<s_a<t}\hat{\mu}_a(s_a,t-1)+I_{s_a,t-1}}}\\
        &
        \le\ell+\sum_{t=1}^{\infty}\sum_{s=1}^{t}\sum_{s_a=\ell}^{t}\mathbb{I}_{\hat{\mu}_{a^*}(s,t)+I_{s,t}\le\hat{\mu}_a(s_a,t)+I_{s_a,t}}.
    \end{align*}
    Note that $\hat{\mu}_{a^*}(s,t)+I_{s,t}\le\hat{\mu}_a(s_a,t)+I_{s_a,t}$ implies that at least one of the following three inequalities is true:
    \begin{align}
        \hat{\mu}_{a^*}(s,t)&\le\mu_{a^*}-I_{s,t}\label{ucb-eq1}\\
        \hat{\mu}_{a}(s_a,t)&\ge\mu_a+I_{s_a,t}\label{ucb-eq2}\\
        \Delta_a&<2I_{s_a,t}\label{ucb-eq3}
    \end{align}
    Otherwise, assume that all three inequalities are false, then we have
    \begin{align*}
        \hat{\mu}_{a^*}(s,t)+I_{s,t}
        &
        >\mu_{a^*}\\
        &
        =\mu_a+\Delta_a\\
        &
        >\hat{\mu}_{a}(s_a,t)-I_{s_a,t}+\Delta_a\\
        &
        \ge\hat{\mu}_{a}(s_a,t)-I_{s_a,t}+2I_{s_a,t}\\
        &
        =\hat{\mu}_{a}(s_a,t)+I_{s_a,t},
    \end{align*}
    which contradicts the condition. By the result of Lemma~\ref{lemma-esterr}, we know that (\ref{ucb-eq1}) or (\ref{ucb-eq2}) hold with probability at most $2t^{-4}$. For $s_a\ge36u^{\frac{1}{v}}\frac{\log(2t^4)\log^{1.5+\frac{1}{v}}T}{\epsilon\Delta_a^{\frac{1+v}{v}}}$, (\ref{ucb-eq3}) is false. Now, let $\ell=\left\lceil36u^{\frac{1}{v}}\frac{\log(2t^4)\log^{1.5+\frac{1}{v}}T}{\epsilon\Delta_a^{\frac{1+v}{v}}}\right\rceil$, we can bound the term $\mathbb{E}[N_a(T)]$ as follows,
    \begin{align*}
        \mathbb{E}[N_a(T)]
        &
        \le\left\lceil36u^{\frac{1}{v}}\frac{\log(2t^4)\log^{1.5+\frac{1}{v}}T}{\epsilon\Delta_a^{\frac{1+v}{v}}}\right\rceil+\sum_{t=1}^{\infty}\sum_{s=1}^{t}\sum_{s_a=\ell}^{t}\mathbb{P}(\text{(\ref{ucb-eq1}) or (\ref{ucb-eq2}) is true})\\
        &
        \le\left\lceil36u^{\frac{1}{v}}\frac{\log(2t^4)\log^{1.5+\frac{1}{v}}T}{\epsilon\Delta_a^{\frac{1+v}{v}}}\right\rceil+\sum_{t=1}^{\infty}\sum_{s=1}^{t}\sum_{s_a=\ell}^{t}\frac{2}{t^4}\\
        &
        \le\left\lceil36u^{\frac{1}{v}}\frac{\log(2t^4)\log^{1.5+\frac{1}{v}}T}{\epsilon\Delta_a^{\frac{1+v}{v}}}\right\rceil+2\sum_{t=1}^\infty\frac{1}{t^2}\\
        &
        \le36u^{\frac{1}{v}}\frac{\log(2t^4)\log^{1.5+\frac{1}{v}}T}{\epsilon\Delta_a^{\frac{1+v}{v}}}+1+\frac{\pi^2}{3}\\
        &\le36u^{\frac{1}{v}}\frac{\log(2t^4)\log^{1.5+\frac{1}{v}}T}{\epsilon\Delta_a^{\frac{1+v}{v}}}+5.
    \end{align*}
    Finally, using that $\mathcal{R}_T=\sum_{a=1}^K(\Delta_a\mathbb{E}N_a(T))$, we directly obtain
    \begin{equation*}
        \mathcal{R}_T\le\sum_{a=1}^K\left(36\frac{\log(2T^4)\log^{1.5+\frac{1}{v}}T}{\epsilon}\left(\frac{u}{\Delta_a}\right)^{\frac{1}{v}}+5\Delta_a\right)
    \end{equation*}
\end{proof}
\begin{proof}[{of Theorem \ref{thm:3}}]
    Consider two adjacent reward streams that differ only on one reward of arm $a$. In each epoch $\tau$, the difference of the mean of arm $a$ between the two adjacent streams is at most $\frac{2B_\tau}{R_\tau}$ since the reward of each arm is truncated by $[-B_\tau, B_\tau]$. Thus, adding noise of ${\rm Lap}(\frac{2B_\tau}{\epsilon R_\tau})$ to $\mu_a$ guarantees $\epsilon$-DP.
\end{proof}
\begin{proof}[{of Lemma \ref{lemma-pulltime}}]
    The bound of $T$ is trivial so we focus on proving the latter bound. Denote the $i$-th reward obtained from arm $a$ by $x_{a,i}$. We first bound $|\mu_a-\bar{\mu}_a|$ and $|\bar{\mu}_a-\widetilde{\mu}_a|$ for each epoch $\tau$ and each $a\in\mathcal{S}$. For $|\mu_a-\bar{\mu}_a|$, by Lemma \ref{alemma3}, Lemma \ref{alemma4} and noting that $\mathbb{E}(X^2\mathbb{I}_{|X|\le B})\le uB^{1-v}$, we have, with probability at least $1-\frac{\beta}{4|\mathcal{S}|\tau^2}$,
    \begin{align*}
        &
        \left| \mu_a-\bar{\mu}_a\right|\\
        &
        =\left| \frac{1}{R_\tau}\sum_{t=1}^{R_\tau}x_{a,t}\mathbb{I}_{\left| x_{a,t}\right|\le B_\tau}-\mu_a\right|\\
        &
        =\left| \frac{1}{R_\tau}\sum_{t=1}^{R_\tau}\left[x_{a,t}\mathbb{I}_{\left| x_{a,t}\right|\le B_\tau}-\mathbb{E}(x_{a,t}\mathbb{I}_{\left| x_{a,t}\right|\le B_\tau})\right]+\frac{1}{R_\tau}\sum_{t=1}^{R_\tau}\left[\mathbb{E}(x_{a,t}\mathbb{I}_{\left| x_{a,t}\right|\le B_e})-\mathbb{E}x_{a,t}\right]\right|\\
        &
        \le \left|\frac{1}{R_\tau}\sum_{t=1}^{R_\tau}\left[x_{a,t}\mathbb{I}_{\left| x_{a,t}\right|\le B_\tau}-\mathbb{E}(x_{a,t}\mathbb{I}_{\left| x_{a,t}\right|\le B_\tau})\right] \right|+\left|\frac{1}{R_\tau}\sum_{t=1}^{R_\tau}\mathbb{E}(x_{a,t}\mathbb{I}_{\left| x_{a,t}\right|> B_\tau})\right|\\
        &
       \le \sqrt{\frac{2B_\tau^{1-v}u\log\left(\frac{4|\mathcal{S}|\tau^2}{\beta}\right)}{R_\tau}}+\frac{B_\tau\log\left(\frac{4|\mathcal{S}|\tau^2}{\beta}\right)}{3R_\tau} +\frac{u}{B_\tau^v}\\
        &\le 4u^{\frac{1}{1+v}}\left(\frac{\log\left(\frac{4|\mathcal{S}|\tau^2}{\beta}\right)}{R_\tau\epsilon}\right)^{\frac{v}{1+v}}=4err_\tau.
    \end{align*}
    For $|\bar{\mu}_a-\widetilde{\mu}_a|$, by using the concentration of the Laplace distribution Lemma \ref{alemma2}, we have with probability at least $1-\frac{\beta}{4|\mathcal{S}|\tau^2}$,
    \begin{equation*}
        \left|\bar{\mu}_a-\widetilde{\mu}_a\right|
        =\left|{\rm Lap}\left(\frac{2B_\tau}{\epsilon R_\tau}\right)\right|
        \le \frac{2B_\tau}{\epsilon R_\tau}\log{\frac{4|\mathcal{S}|\tau^2}{\beta}}
        =2u^{\frac{1}{1+v}}\left(\frac{\log\left(\frac{4|\mathcal{S}|\tau^2}{\beta}\right)}{R_\tau\epsilon}\right)^{\frac{v}{1+v}}=2err_\tau.
    \end{equation*}
     Given an epoch $\tau$, we denote by $\mathcal{E}_\tau$ the event where for all $i\in\mathcal{S}$ it holds that $|\mu_a-\bar{\mu}_a|\le4err_\tau$ and $|\bar{\mu}_a-\widetilde{\mu}_a|\le 2err_\tau$ and denote $\mathcal{E}=\bigcup_{\tau\ge 1}\mathcal{E}_\tau$. By taking the union bound, we have $$\mathbb{P}(\mathcal{E}_\tau)\ge1-\frac{\beta}{2\tau^2}$$
     and
     $$\mathbb{P}(\mathcal{E})\ge1-\frac{\beta}{2}\left(\sum_{\tau\ge 1}\tau^{-2}\right)\ge1-\beta.$$
     In the remainder of the proof, we assume that $\mathcal{E}$ holds. So for any epoch $\tau$ and any viable arm $a$, we have $|\widetilde{\mu}_a-\mu_a|\le6err_\tau$. As a result, for any epoch $\tau$ and any two arms $a$ and $a^\prime$, we have $$\left|(\widetilde{\mu}_a-\widetilde{\mu}_{a^\prime})-(\mu_a-\mu_{a^\prime})\right|\le 12err_\tau.$$

     Next, we show that under $\mathcal{E}$, the optimal arm $a^*$ is never eliminated. For any epoch $\tau$, let $a_\tau=\arg\max_{a\in\mathcal{S}}{\widetilde{\mu}_a}$. Since $$\left|(\widetilde{\mu}_{a_\tau}-\widetilde{\mu}_{a^*})-(\mu_{a_\tau}-\mu_{a^*})\right|=|(\widetilde{\mu}_{a_\tau}-\widetilde{\mu}_{a^*})+\Delta_{a_\tau}|=(\widetilde{\mu}_{a_\tau}-\widetilde{\mu}_{a^*})+\Delta_{a_\tau}\le 12err_\tau,$$
     it is easy to see that the algorithm doesn't eliminate $a^*$.

     Next, we show that under $\mathcal{E}$, in any epoch $\tau$ the algorithm eliminate all viable arms with sub-optimality gap at least $D_\tau=2^{-\tau}$. Fix an epoch $\tau$ and a viable arm $a$ with sup-optimality gap $\Delta_a\ge D_\tau$. Due to $R_\tau=u^{\frac{1}{v}}\left(\frac{24^{(1+v)/v}\log\left(4\left|\mathcal{S}\right|\tau^2/\beta\right)}{\epsilon D_\tau^{(1+v)/v}}\right)+1$, we know that $12err_\tau<\frac{D_\tau}{2}$. Thus, $$\widetilde{\mu}_{a_\tau}-\widetilde{\mu}_a\ge\widetilde{\mu}_{a^*}-\widetilde{\mu}_a\ge\Delta_a-12err_\tau>D_\tau-\frac{D_\tau}{2}=\frac{D_\tau}{2}>12err_\tau,$$
     which means arm $a$ is eliminated by the algorithm.

     Finally, fix a suboptimal arm $a$, we derive the upper bound on the total number of timesteps that arm $a$ is pulled. Let $\tau(a)$ be the first epoch such that $\Delta_a\ge D_{\tau(a)}$, which implying $D_{\tau(a)}\le\Delta_a\le D_{\tau(a)-1}=2 D_{\tau(a)}$. Due to $\Delta_a\le2 D_{\tau(a)}$ and $D_{\tau(a)}=2^{-\tau(a)}$, we have $\tau(a)\le\log_2\left(\frac{2}{\Delta_a}\right)$. Thus, the total number of pulls of arm $a$ is
     \begin{align*}
       \sum_{\tau\le \tau(a)}{R_\tau}&\le\sum_{\tau\le \tau(a)}{2^{\frac{1+v}{v}[\tau-\tau(a)]}}R_{\tau(a)} \le R_{\tau(a)}\sum_{i=0}^{\tau(a)-1}{2^{-\frac{1+v}{v}i}} \\
         &\le\frac{1}{1-2^{-\frac{1+v}{v}}}R_{\tau(a)}\\
         &
         =\frac{2^{\frac{1+v}{v}}}{2^{\frac{1+v}{v}}-1}\left(u^{\frac{1}{1+v}}\frac{24^\frac{1+v}{v}}{D_\tau^{\frac{1+v}{v}}}\frac{\log\left(\frac{4|\mathcal{S}|\tau(a)^2}{\beta}\right)}{\epsilon}+1\right)\\
         &\le u^{\frac{1}{1+v}}\frac{48^\frac{1+v}{v}}{D_\tau^{\frac{1+v}{v}}}\frac{\log\left(\frac{4|\mathcal{S}|\tau(a)^2}{\beta}\right)}{\epsilon}+2\\
         &
         \le u^{\frac{1}{1+v}}\left(\frac{96}{\Delta_a}\right)^{\frac{1+v}{v}}\frac{\log\left(\frac{4K}{\beta}\right)+\log\log\left(\frac{2}{\Delta_a}\right)}{\epsilon}+2 \\
         &= O\left(  u^{\frac{1}{1+v}}\left(\frac{1}{\Delta_a}\right)^{\frac{1+v}{v}}\frac{\log\left(\frac{4K}{\beta}\right)+\log\log\left(\frac{2}{\Delta_a}\right)}{\epsilon}\right)
     \end{align*}
      where the last inequality is due to the bounds $D_\tau>\Delta_a/2$, $|\mathcal{S}|\le K$, $\tau(a)\le\log_2(2/\Delta_a)$ and $K\le2$. Combining with the trivial upper bound $T$, the lemma follows directly.
\end{proof}

\begin{proof}[{of Theorem \ref{thm:4}}]
We first consider the instance-dependent expected regret.
Denote the total number of rounds to pull each sub-optimal arm $a\ne a^*$ by $T_a$. Taking $\beta=\frac{1}{T}$, then with probability at least $1-\frac{1}{T}$, we have
    \begin{align*}
        \mathcal{R}_T
        &
        \le\sum_{\Delta_a>0}{T_a}\cdot\Delta_a=\sum_{\Delta_i>0}{\left[\frac{u^{\frac{1}{1+v}}}{\epsilon(\Delta_a)^{\frac{1}{v}}}\left(\log\left(\frac{K}{\beta}\right)+\log\log\left(\frac{1}{\Delta_a}\right)\right)\right]}\\
        &
        ={O}\left(\frac{u^{\frac{1}{1+v}}\log T}{\epsilon}\sum_{\Delta_a>0}\left(\frac{1}{\Delta_a}\right)^\frac{1}{v}\right).
    \end{align*}
    With probability at most $\frac{1}{T}$, Algorithm~\ref{Algo-DPSE} will fail to identify the optimal arm and thus incur an expected cumulative $regret$ of $O(T\cdot\max_{\Delta_a>0}\Delta_a)$. Combining the two cases, we obtain
    \begin{align*}
        \mathcal{R}_T
        &
        \le(1-\frac{1}{T})\cdot {O}\left[\frac{u^{\frac{1}{1+v}}\log T}{\epsilon}\sum_{\Delta_a>0}\left(\frac{1}{\Delta_a}\right)^\frac{1}{v}\right]+\frac{1}{T}\cdot O(T\cdot\max_{\Delta_a>0}\Delta_a)\\
        &
        \le\widetilde{O}\left(\frac{u^{\frac{1}{1+v}}\log T}{\epsilon}\sum_{\Delta_a>0}{\left(\frac{1}{\Delta_a}\right)^{\frac{1}{v}}}+\max_{\Delta_a>0}\Delta_a\right).
    \end{align*}

We now consider the instance-independent expected regret, which is inspired by \citep{sajed2019optimal}.

    Throughout the proof we assume Algorithm~\ref{Algo-DPSE} runs with a parameter $\beta=\frac{1}{T}$. Since any arm $a$ with $\Delta_a<\frac{1}{T}$ yields a negligible expected regret bound of at most $1$, we assume $\Delta_a\ge\frac{1}{T}$. Then the bound of Lemma~\ref{lemma-pulltime} becomes
    $$\min\left\{T, C\frac{u^{\frac{1}{1+v}}\log (KT)}{\epsilon}\left(\frac{1}{\Delta_a}\right)^{\frac{1+v}{v}}\right\}.$$
    It follows that for any suboptimal arm $a$, the expected regret from pulling arm $a$ is therefore at most $$\min\left\{\Delta_aT, C\frac{u^{\frac{1}{1+v}}\log (KT)}{\epsilon}\left(\frac{1}{\Delta_a}\right)^{\frac{1}{v}}\right\}.$$

    Denote by $\Delta^*$ the gap which equates the two possible regret bounds when all arms are pulled $T/K$ times. That is, $\Delta^*\frac{T}{K}=C\frac{u^{\frac{1}{1+v}}\log T}{\epsilon}\left(\frac{1}{\Delta^*}\right)^{\frac{1}{v}}$. It can be easily derived that
    $$\Delta^*=O\left(u^{\frac{v}{(1+v)^2}}\left(\frac{CK}{\epsilon}\frac{\log T}{T}\right)^{\frac{v}{1+v}}\right).$$
    Note that in the setting where all suboptimal arms have the gap of precisely $\Delta^*$, the expected regret bound is proportional to
    $$O\left(u^{\frac{v}{(1+v)^2}}\left(\frac{CK\log T}{\epsilon}\right)^{\frac{v}{1+v}}T^{\frac{1}{1+v}}\right).$$

    Next, we show that no matter how different the arm gaps are, the expected regret of Algorithm~\ref{Algo-DPSE} is still proportional to this bound. For an arbitrary MAB instance, we rearrange arm by the increasing order of arm gaps such that arm $1$ is the optimal arm. We partition the set of suboptimal arms $2,\cdots,K$ to two sets: $\{2,\cdots,k^\prime\}$ and $\{k^\prime+1,\cdots K\}$, where $k^\prime$ is the largest index of the arm with gap at most $\Delta^*$. If we only pull the arms in the former set, the upper bound on the expected regret will be $T\Delta^*$, since the incurred expected regret at each pull is at most $\Delta^*$. If we only pull the arms in the latter set, the expected regret will be at most
    \begin{align*}
        &
        C\frac{u^{\frac{1}{1+v}}\log (KT)}{\epsilon}\sum_{a=k^\prime+1}^{K}\left(\frac{1}{\Delta_a}\right)^{\frac{1}{v}}
        \le 2C\frac{u^{\frac{1}{1+v}}\log T}{\epsilon}\sum_{a=k^\prime+1}^{K}\left(\frac{1}{\Delta^*}\right)^{\frac{1}{v}}\\
        &
        =2(K-k^\prime)\cdot C\frac{u^{\frac{1}{1+v}}\log T}{\epsilon}\left(\frac{1}{\Delta^*}\right)^{\frac{1}{v}}=2(K-k^\prime)\Delta^*\frac{T}{K}\le 2T\Delta^*.
    \end{align*}
    Since $\{2,\cdots,k^\prime\}$ and $\{k^\prime+1,\cdots K\}$ is a partition of the subtoptimal arms, one of the two sets contributes at least half of the expected regret. It is simple to see that the expected regret is upper bounded by $O(T\Delta^*)=O\left(u^{\frac{v}{(1+v)^2}}\left(\frac{CK\log T}{\epsilon}\right)^{\frac{v}{1+v}}T^{\frac{1}{1+v}}\right)$.
\end{proof}

\begin{theorem}\label{thm:6}
   There exists a heavy-tailed two-armed bandit problem with arm $2$ being sub-optimal, $u\le1$ in (\ref{eq:2}), and $\Delta\triangleq\mu_1-\mu_2\in(0,\frac{1}{5})$. Such that for any $\epsilon$-DP algorithm $\mathcal{A}$ with expected regret at most $T^{\frac{3}{4}}$ for sufficiently large number of rounds $T$,\footnote{Note that we can replace $\frac{3}{4}$ to other constants. The same to other results.} we have
    \begin{equation}
        \mathcal{R}_T\ge\Omega\left(\frac{\log T}{\epsilon}(\frac{1}{\Delta})^{\frac{1}{v}}\right).
    \end{equation}
\end{theorem}

\begin{proof}[{of Theorem \ref{thm:6}}]
    Let $\gamma=(5\Delta)^{\frac{1}{v}}$. Consider the instance $\bar{P}$ where the distribution of arm $1$ is $$\nu_1=\left(1-\frac{\gamma^{1+v}}{2}\right)\delta_0+\frac{\gamma^{1+v}}{2}\delta_{1/\gamma}$$
    and the distribution of arm $2$ is $$\nu_2=\left[1-\left(\frac{\gamma^{1+v}}{2}-\Delta\gamma\right)\right]\delta_0+\left(\frac{\gamma^{1+v}}{2}-\Delta\gamma\right)\delta_{1/\gamma},$$
    where $\delta_x$ is the Dirac distribution on $x$ and the distribution $p\cdot\delta_x+(1-p)\cdot\delta_y$ takes the value $x$ with probability $p$ and the value $y$ with probability $1-p$. It is easy to verify that $$\mathbb{E}[\nu_1]=\frac{5}{2}\Delta, u(\nu_1)=\frac{1}{2}<1$$
    and $$\mathbb{E}[\nu_2]=\frac{3}{2}\Delta, u(\nu_2)=\frac{3}{10}\le1.$$
    Denote $\mathcal{E}$ as the event that arm $2$ is pulled at most $t_2\triangleq\frac{\log T}{100\epsilon\cdot 5^{1/v}}\left(\frac{1}{\Delta}\right)^{\frac{1+v}{v}}$ times. We show in the following that $\mathbb{P}_{\mathcal{A},\bar{P}}(\mathcal{E})\le\frac{1}{2}$. Consider another instance $\bar{Q}$ where the distribution of arm $1$ remains unchanged, the distribution of arm $2$ is
    $$\nu_2^\prime=\left[1-\left(\frac{\gamma^{1+v}}{2}+\Delta\gamma\right)\right]\delta_0+\left(\frac{\gamma^{1+v}}{2}+\Delta\gamma\right)\delta_{1/\gamma}.$$
    Note that since $\Delta\in(0,\frac{1}{5})$, $(\frac{\gamma^{1+v}}{2}+\Delta\gamma)<1$, hence the instance $\bar{Q}$ is reasonable. It is also easy to verify that $$\mathbb{E}[\nu_2^\prime]=\frac{7}{2}\Delta, u(\nu_2^\prime)=\frac{7}{10}\le1.$$ Denote the regret of algorithm $\mathcal{A}$ under the instance $\bar{Q}$ by $\mathcal{R}_{T,\bar{Q}}^{\mathcal{A}}$. Then we have
    $$\mathcal{R}_{T,\bar{Q}}^{\mathcal{A}}\ge\mathbb{P}_{\mathcal{A},\bar{Q}}(\mathcal{E})(T-t_2)\Delta\ge\frac{T\Delta}{2}\mathbb{P}_{\mathcal{A}, \bar{Q}}(\mathcal{E}),$$
    where the first inequality is since that when $\mathcal{E}$ holds, we have additional $(T-t_2)\Delta$ regret and the second inequality is since that we assume $T$ is sufficiently large. Recall that $\mathcal{R}_T^{\mathcal{A}}$ is at most $T^{\frac{3}{4}}$, then we can obtain that $$\mathbb{P}_{\mathcal{A},\bar{Q}}(\mathcal{E})\le\frac{2}{\Delta T^{\frac{1}{4}}}.$$

    Now we consider the influence of differential privacy. Before that, we recall the following lemma.
    \begin{lemma}[Lemma 6.1 in~\citep{KarwaV_ITCS18}]\label{lemma-grpri}
    For each pair of distribution $\mathcal{D}$ and $\mathcal{D}^\prime$ and any $\epsilon$-differentially private mechanism $\mathcal{M}$, let $\mathbb{P}$ and $\mathbb{P}^\prime$ be the two marginal distributions on the output of $\mathcal{M}$ evaluated on $n$ data sampled i.i.d. from $\mathcal{D}$ and $\mathcal{D}^\prime$ respectively, then for any event $E$, we have
    \begin{equation}
        \mathbb{P}[E]\le e^{6\epsilon n\cdot d_{TV}(\mathcal{D}, \mathcal{D}^\prime)}\mathbb{P}^\prime[E],
    \end{equation}
    where $d_{TV}(\mathcal{D}, \mathcal{D}^\prime)$ is the total-variation distance between $\mathcal{D}$ and $\mathcal{D}^\prime$.
\end{lemma}

    Lemma~\ref{lemma-grpri} suggests that the ``effective" group privacy for the case that $n$ data items of the inputs are drawn i.i.d. either from distribution $\mathcal{D}$ or from distribution $\mathcal{D}^\prime$ is proportional to $\exp(6\epsilon n\cdot d_{TV}(\mathcal{D},\mathcal{D}^\prime))$. We apply the coupling argument in~\citep{KarwaV_ITCS18} to our setting. Note that we only consider the change under the event $\mathcal{E}$ here. Suppose there is an oracle $\mathcal{O}$ that can generate a collection of at most $t_2$ pairs of data, where the left ones are i.i.d. samples from $\nu_2$ and the right ones are i.i.d. samples from $\nu_2^\prime$. Whenever the algorithm needs to sample a reward from arm $2$, it turns to the oracle $\mathcal{O}$ and $\mathcal{O}$ provides either a fresh left-sample or a right-sample depending on the true environment (the true reward distribution of arm $2$). Suppose there is a counter $\mathcal{C}$ standing between the algorithm $\mathcal{A}$ and the Oracle $\mathcal{O}$. And if $\mathcal{O}$ runs out of $t_2$ samples, $\mathcal{C}$ routes $\mathcal{A}$'s oracle calls to another oracle. Lemma~\ref{lemma-grpri} guarantees that, the oracle never runs out of $t_2$ samples, i.e. event $E$ happens, with similar probabilities under $\nu_2$ and $\nu_2^\prime$. Formally, using the result of Lemma~\ref{lemma-grpri}, for sufficiently large $T$ such that $T^{0.13}>\frac{4}{\Delta}$, we have
    \begin{align*}
        \mathbb{P}_{\mathcal{A},\bar{P}}(E)
        &
        \le e^{6\epsilon t_2d_{TV}\left(\nu_2, \nu_2^\prime\right)}\cdot\mathbb{P}_{\mathcal{A},\bar{Q}}(E)
        \le e^{6\epsilon t_2\cdot 2\Delta\gamma}\cdot\frac{2}{\Delta T^{\frac{1}{4}}}\\
        &
        \le e^{\frac{12}{100}\log T}\cdot \frac{2}{\Delta}T^{-\frac{1}{4}}
        =T^{-0.13}\frac{2}{\Delta}<\frac{1}{2}.
    \end{align*}
    Thus we obtain
    \begin{equation*}
        \mathcal{R}_T\ge\Omega\left(\Delta\cdot t_2\right)\ge\Omega\left(\frac{\log T}{\epsilon}\left(\frac{1}{\Delta}\right)^{\frac{1}{v}}\right).
    \end{equation*}
\end{proof}

\begin{proof}[{of Theorem \ref{thm:7}}]
    We focus on the $K$ arms with mean reward satisfying $\frac{1}{2}\ge\mu_1\ge\cdots\ge\mu_K$. Consider the instance $\bar{P}$ where the distribution for each arm $a\in[K]$ is $$\nu_a=\left(1-\frac{s_a^{1+v}}{2}\right)\delta_0+\frac{s_a^{1+v}}{2}\delta_{1/s_a},$$
    where $s_a=(2\mu_a)^{\frac{1}{v}}$. It is easy to verify for each $a\in[K]$ that $$\mathbb{E}[\nu_a]=\mu_a, u(\nu_a)=\frac{1}{2}<1.$$
    Now, we fix an arm $a\ne1$ and denote $\mathcal{E}_a$ as the event that the arm $a$ is pulled at most $t_a\triangleq\frac{\log T}{100\epsilon 4^{1/v}}\left(\frac{1}{\Delta_a}\right)^\frac{1+v}{v}$ times. We show in the following that $\mathbb{P}_{\mathcal{A},\bar{P}}(\mathcal{E}_a)\le\frac{1}{2K}$.
    Consider another instance $\bar{Q}_a$ where the distribution of any arm $a^\prime\ne a$ remains unchanged, and the distribution of arm $a$ is $$\nu_a^\prime=\left[ 1-\left(\frac{s_a^{1+v}}{2}+2\Delta_a\gamma_a\right)\right]\delta_0+\left(\frac{s_a^{1+v}}{2} \right)\delta_{1/s_a}+(2\Delta_a\gamma_a)\delta_{1/\gamma_a},$$
    where $\gamma_a=(4\Delta_a)^{\frac{1}{v}}$. Note that since $\mu_a\le\frac{1}{6}$ and $\Delta_a\le\frac{1}{12}$, we have  $\mu_a^{1+v}<\frac{1}{6}$ and $\Delta_a^{1+v}<\frac{1}{12}$, and then
    \[\frac{s_a^{1+v}}{2}+2\Delta_a\gamma_a=2^{\frac{1}{v}}\mu_a^{\frac{1+v}{v}}+2\cdot4^{\frac{1}{v}}\cdot\Delta_a^{\frac{1+v}{v}}<\left(\frac{1}{3}\right)^{\frac{1}{v}}+2\cdot\left(\frac{1}{3}\right)^{\frac{1}{v}}<1,\]
    thus the postulated $\nu_a^\prime$ is reasonable. It is also easy to verify that $$\mathbb{E}[\nu_a^\prime]=\mu_a+2\Delta_a=\mu_1+\Delta_a, u(v_a^\prime)=1.$$
    Then for sufficiently large $T$ we have
    $$\mathcal{R}_{T,\bar{Q}_a}^{\mathcal{A}}\ge\mathbb{P}_{\mathcal{A},\bar{Q}_a}(\mathcal{E}_a)\cdot(T-t_a)\cdot\Delta_a\ge\frac{T\Delta_a}{2}\mathbb{P}_{\mathcal{A},\bar{Q}_a}(\mathcal{E}_a).$$
    Combining with $\mathcal{R}_{T}^{\mathcal{A}} \le T^{\frac{3}{4}}$, we have $$\mathbb{P}_{\mathcal{A},\bar{Q}_a}(\mathcal{E}_a)\le\frac{2}{\Delta_a\cdot T^{\frac{1}{4}}}.$$
    By lemma~\ref{lemma-grpri}, for sufficiently large $T$ such that $T^{0.13}>\max\limits_{\Delta_a>0}\frac{4K}{\Delta_a}$, we have
    \begin{align*}
        \mathbb{P}_{\mathcal{A},\bar{P}}(\mathcal{E}_a)
        &
        \le e^{6\epsilon t_ad_{TV}\left(\nu_a, \nu_a^\prime\right)}\cdot\mathbb{P}_{\mathcal{A},\bar{Q}_a}(\mathcal{E}_a)
        \le e^{6\epsilon t_a\cdot 2\Delta_a\gamma_a}\cdot\frac{2}{\Delta_a T^{\frac{1}{4}}}\\
        &
        \le e^{\frac{12}{100}\log T}\cdot \frac{2}{\Delta_a}T^{-\frac{1}{4}}
        =T^{-0.13}\frac{2}{\Delta_a}<\frac{1}{2K}.
    \end{align*}
    Then, with probability at least $1-K\cdot\frac{1}{2K}=\frac{1}{2}$, $\mathcal{A}$ will pull each $a\ne 1$ at least $t_a$ times. Thus we obtain
    \begin{equation*}
        \mathcal{R}_T\ge\Omega\left(\sum_{\Delta_a>0}{\Delta_a\cdot t_a}\right)\ge\Omega\left(\frac{\log T}{\epsilon}\sum_{\Delta_a>0}{\left(\frac{1}{\Delta_a}\right)^{\frac{1}{v}}}\right).
    \end{equation*}
\end{proof}

\section{OMITTED PROOFS FOR SECTION~\ref{sec:LDP} (local differential privacy)}
\label{sec:appendix-LDP}
\begin{proof}[{of Theorem \ref{DPofLDP}}]
Since each $|\tilde{x}_{a, r}|$ is bounded by $ B_r$. Thus, adding noise of Lap($\frac{B_r}{\epsilon}$) to $\tilde{x}_{a, r}$ guarantees $\epsilon$-LDP.
\end{proof}


\begin{proof}[{of Theorem \ref{UpperBoundLDP}}]
Similar to Lemma \ref{lemma-esterr}, we prove the following lemma.
\begin{lemma}
 For any instance of the $K$-armed MAB problem, denote by $a^*$ the optimal arm and by $\Delta_a$ the gap between the mean of arm $a^*$ and any sub-optimal arm $a\ne a^*$. Fix the time horizon $T$ and confidence level $\beta \in (0,1)$. Then, with probability at least $1-\beta$, in Algorithm~\ref{Algo-LDPSE}, the total number of rounds to pull each sub-optimal arm $a\ne a^*$, denoted by $T_a$, is at most
    \begin{equation}
        \min \left\{T, O\left( \frac{u^{\frac{2}{v}}}{\epsilon^2(\Delta_a)^{\frac{1+v}{v}}}\left(\log\Big(\frac{K}{\beta}\Big)+\log\log\Big(\frac{1}{\Delta_a}\Big)\right)\right)\right\}.
    \end{equation}
\end{lemma}

We first bound the error of $|\mu_a-\widetilde{\mu}_a|$ for each epoch $\tau$ and each arm $a$. Recall that
\begin{equation}\label{DecompMu}
\left|\widetilde{\mu}_{a}-\mu_{a}\right|=\left|\frac{\sum_{i=1}^{R_\tau}\widetilde{x}_{a, i}}{R_\tau}-\mu_a \right|+\left|\frac{\sum_{i=1}^{R_\tau} Y_{a, i}}{R_\tau}\right|
\end{equation}
where $Y_{a, i} \sim {\rm Lap}\left(2\frac{B_\tau}{\epsilon}\right) $ and $B_\tau=\left(\frac{u\sqrt{R_\tau}\epsilon}{\sqrt{\log\left(8\left|\mathcal{S} \right|\tau^2/\beta\right)}}\right)^{\frac{1}{1+v}}$.
According to Heoffding bound (Lemma \ref{alemma5}), we can get with probability $1-\delta$,$$
\left|\frac{\sum_{i=1}^{R_\tau} Y_{a, i}}{R_\tau}\right| \leq \frac{4\sqrt{\log \frac{2}{\delta}}}{\sqrt{R_\tau}\epsilon}\left(\frac{u\sqrt{R_\tau}\epsilon}{\sqrt{\log\left(8\left|\mathcal{S} \right|\tau^2/\beta\right)}}\right)^{\frac{1}{1+v}}.$$

Setting $\delta= \frac{\beta}{4\left|\mathcal{S} \right|e^2}$, we have $$
\left|\frac{\sum_{i=1}^{R_\tau} Y_{a, i}}{R_\tau}\right| \leq 4u^{\frac{1}{1+v}}\left(\frac{\sqrt{\log\left(8\left|\mathcal{S} \right|\tau^2/\beta\right)}}{\sqrt{R_\tau}\epsilon}\right)^{\frac{v}{1+v}}.$$
Now we consider the first term on the right side of Equation (\ref{DecompMu}). From Lemma \ref{alemma3}, Lemma \ref{alemma4},
noting that $\mathbb{E}\left(X^{2} \mathbb{I}_{|X| \leq B}\right) \leq u B^{1-v}$, we have, with probability at least $1-\delta$
\begin{equation}\nonumber
\begin{aligned}
& \mu_a-\frac{1}{R_\tau} \sum_{i=1}^{R_\tau} \widetilde{x}_{a, i} \\ &=\frac{1}{R_\tau} \sum_{i=1}^{R_\tau}\left(\mu_a-\mathbb{E}\left(X \cdot \mathbb{I}_{|X| \leq B_\tau}\right)\right)+\frac{1}{R_\tau} \sum_{i=1}^{R_\tau}\left(\mathbb{E}\left(X \cdot \mathbb{I}_{|X| \leq B_\tau}\right)-x_{a, i} \mathbb{I}_{\left|x_{a, i}\right| \leq B_\tau}\right) \\
&=\frac{1}{R_\tau} \sum_{i=1}^{R_\tau} \mathbb{E}\left(X \cdot \mathbb{I}_{|X|>B_\tau}\right)+\frac{1}{R_\tau} \sum_{i=1}^{R_\tau}\left(\mathbb{E}\left(X \cdot \mathbb{I}_{|X| \leq B_\tau}\right)-x_{a, i}\mathbb{I}_{\left|x_{a, i}\right| \leq B_\tau}\right) \\
& \leq \frac{u}{B_\tau^{v}}+\sqrt{\frac{2 B_\tau^{1-v} u \log \left(\delta^{-1}\right)}{R_\tau}}+\frac{B_\tau \log \left(\delta^{-1}\right)}{3 R_\tau}.
\end{aligned}
\end{equation}
Taking $\delta=\frac{\beta}{8\left|\mathcal{S} \right|\tau^2}$, we have
\begin{align*}
    \frac{u}{B_\tau^{v}}\leq u^{\frac{1}{1+v}}\left(\frac{\sqrt{\log\left(8\left|\mathcal{S} \right|\tau^2/\beta\right)}}{\sqrt{R_\tau}\epsilon}\right)^{\frac{v}{1+v}},
\end{align*}
\begin{align*}
    \sqrt{\frac{2 B_\tau^{1-v} u \log \left(\delta^{-1}\right)}{R_\tau}}\leq   \sqrt{\frac{2 B_\tau^{1-v} u \log \left(\delta^{-1}\right)}{\sqrt{R_\tau}}}\sqrt{\frac{\sqrt{\log \frac{1}{\delta}}}{\sqrt{R_\tau}}}\leq  \sqrt{2}u^{\frac{1}{1+v}}\left(\frac{\sqrt{\log\left(8\left|\mathcal{S} \right|\tau^2/\beta\right)}}{\sqrt{R_\tau}\epsilon}\right)^{\frac{v}{1+v}},
\end{align*}
where the last inequality is due to the fact that $R_\tau \geq \log\left(8\left|\mathcal{S} \right|\tau^2/\beta\right).$ Moreover we have
\begin{align*}
    \frac{B_\tau \log \left(\delta^{-1}\right)}{3 R_\tau}\leq   \frac{B_\tau \sqrt{\log \left(\delta^{-1}\right)}}{3\sqrt{R_\tau}} \frac{\sqrt{\log \left(\delta^{-1}\right)}}{\sqrt{R_\tau}}\leq \frac{1}{3} u^{\frac{1}{1+v}}\left(\frac{\sqrt{\log\left(8\left|\mathcal{S} \right|\tau^2/\beta\right)}}{\sqrt{R_\tau}\epsilon}\right)^{\frac{v}{1+v}}.
\end{align*}
Thus, in total we have $\left|\frac{\sum_{i=1}^{R_\tau}\widetilde{X}_{a,i}}{R_\tau}-\mu_k \right| \leq 3u^{\frac{1}{1+v}}\left(\frac{\sqrt{\log\left(8\left|\mathcal{S} \right|\tau^2/\beta\right)}}{\sqrt{R_\tau}\epsilon}\right)^{\frac{v}{1+v}}.$ Taking union bound yields that with probability at least $1-\frac{3\beta}{8\left|\mathcal{S} \right|\tau^2}$,
\begin{equation}\label{eq-perarm}
    \left|\widetilde{\mu}_{a}-\mu_{a}\right| \leq 7u^{\frac{1}{1+v}}\left(\frac{\sqrt{\log\left(8\left|\mathcal{S} \right|\tau^2/\beta\right)}}{\sqrt{R_\tau}\epsilon}\right)^{\frac{v}{1+v}}.
\end{equation}
Denote by $\mathcal{E}_\tau$ the event where for all arms $a\in\mathcal{S}$, (\ref{eq-perarm}) holds and denote $\mathcal{E}=\cup_{\tau\ge1}\mathcal{E}_{\tau}$. Taking the union for all epochs and arms, we have $\mathcal{E}$ holds w.p. $1-\frac{3\beta}{8}\sum_{\tau \geq 1} \tau^{-2} \geq 1-\beta. $ As a result, for any epoch $\tau$ and any two arms $i,j \in \mathcal{S}$ we have that $$\left| (\widetilde{\mu}_i-\widetilde{\mu}_j)-(\mu_i-\mu_j)\right| \leq 14u^{\frac{1}{1+v}}\left(\frac{\sqrt{\log\left(8\left|\mathcal{S} \right|\tau^2/\beta\right)}}{\sqrt{R_\tau}\epsilon}\right)^{\frac{v}{1+v}}.$$

Next, we show that under $\mathcal{E}$, in each epoch, the optimal arm $a^*$ will not be eliminated. Let $a_{\tau}=\operatorname{argmax}_{a \in \mathcal{S}} \widetilde{\mu}_{a}$, then in the epoch $\tau$, $$\widetilde{\mu}_{a_\tau}-\widetilde{\mu}_{a^*}+\Delta_{a_\tau} \leq14u^{\frac{1}{1+v}}\left(\frac{\sqrt{\log\left(8\left|\mathcal{S} \right|\tau^2/\beta\right)}}{\sqrt{R_\tau}\epsilon}\right)^{\frac{v}{1+v}},$$ so the algorithm doesn't eliminate $\mu^*$.

Next, we show that under $\mathcal{E}$, in each epoch $\tau$, we eliminate all arms with sub-optimality gap $\geq 2^{-\tau} = D_\tau$. Fix a sub-optimal arm $a$ such that $\Delta_a \geq D_\tau$. By the definition of $R_\tau$, We know that $14u^{\frac{1}{1+v}}\left(\frac{\sqrt{\log\left(8\left|\mathcal{S} \right|\tau^2/\beta\right)}}{\sqrt{R_\tau}\epsilon}\right)^{\frac{v}{1+v}} \leq \frac{D_\tau}{2}$. Therefore, since arm $a^*$ remains viable, we have that \[\begin{aligned}
\widetilde{\mu}_{a_\tau}-\widetilde{\mu}_a \geq \widetilde{\mu}_{a^*}-\widetilde{\mu}_a &\geq \Delta_a-14u^{\frac{1}{1+v}}\left(\frac{\sqrt{\log\left(8\left|\mathcal{S} \right|\tau^2/\beta\right)}}{\sqrt{R_\tau}\epsilon}\right)^{\frac{v}{1+v}}\\
&\geq \Delta_\tau- \frac{\Delta_\tau}{2} \geq \frac{\Delta_\tau}{2} \geq 14u^{\frac{1}{1+v}}\left(\frac{\sqrt{\log\left(8\left|\mathcal{S} \right|\tau^2/\beta\right)}}{\sqrt{R_\tau}\epsilon}\right)^{\frac{v}{1+v}},
\end{aligned}\]
which ensures that arm $a$ is removed from $\mathcal{S}$.

Lastly, for any fixed sub-optimal arm $a$, let $\tau(a)$ be the first epoch s.t. $\Delta_a \geq D^2_{\tau(a)}$, implying $D^2_{\tau(a)} \leq \Delta_a < D^2_{\tau(a)-1}=4D^2_{\tau(a)}$. For any epoch $\tau$, we have $R_{\tau+1} \geq 4^{-\frac{1+v}{v}} R_\tau$, we have that the total number of pulls of arm $a$ is
\begin{equation}\nonumber
\begin{aligned}
\sum_{\tau \leq \tau(a)} R_{\tau} &\leq \sum_{\tau \leq \tau(a)} \left(4^{-\frac{1+v}{v}}\right)^{\tau-\tau(a)} R_{\tau(a)}
\\&\leq R_{\tau(a)} \sum_{i \geq 0} \left(4^{- \frac{1+v}{v}}\right)^i \\
&\leq \frac{1}{1-4^{- \frac{1+v}{v}}}\left(\frac {28^{\frac{2(1+v)}{v}} \log \left(8\left|\mathcal{S} \right|\tau^2(a)/\beta\right)}{\epsilon^2D_{\tau(a)}^{\frac{2(1+v)}{v}}}u^{\frac{2}{v}}+\log \left(8\left|\mathcal{S} \right|\tau^2(a)/\beta\right)\right)\\
&\leq O\left(\frac { \log \left(\left|\mathcal{S} \right|\tau^2(a)/\beta\right)}{\epsilon^2\Delta_{a}^{\frac{1+v}{v}}}u^{\frac{2}{v}}\right)
\end{aligned}
\end{equation}
Note that $\tau(a)=O(\log \frac{1}{\Delta_a})$, thus the algorithm pulls sub-optimal arm $a$ for a number of timesteps is bounded by $ O\left( \frac {\left(\log T + \log \log \frac{1}{\Delta_a} \right)}{\epsilon^2\Delta_{a}^{\frac{1+v}{v}}}u^{\frac{2}{v}}\right)$ with probability $1-\frac{1}{T}.$ Thus the regret is bounded by
\begin{equation*}
    \mathcal{R}_T\leq (1-\frac{1}{T}) O\left( \frac {\left(\log T + \log \log \frac{1}{\Delta_a} \right)}{\epsilon^2\Delta_{a}^{\frac{1}{v}}}u^{\frac{2}{v}}\right)+\frac{1}{T} T\max_{a}\Delta_a.
\end{equation*}
The proof of the instance-independent regret is almost the same as the proof of Theorem \ref{thm:4}, we omit it here.
\end{proof}


Before proofing the lower bounds in the LDP model, we first recall the two useful lemmas provided by \citep{basu2019differential}. Let $\mathcal{H}_T\triangleq\{(a_i,x_i)\}_{i=1}^T$ be the observed history produced by the interaction between the algorithm $\mathcal{A}$ and bandit problem instance $\mathcal{P}$ up to round $T$. Obviously, an observed history $\mathcal{H}_T$ is a random variable sampled from the measurable space $(([K]\times\mathbb{R})^T, \mathcal{B}(([K]\times\mathbb{R})^T), \mathbb{P}_{\mathcal{A}P})$, where $\mathcal{B}(([K]\times\mathbb{R})^T)$ is the Borel set on $([K]\times\mathbb{R})^T$ and $\mathbb{P}_{\mathcal{A}P}$ is the probability measure induced by the algorithm $A$ and the instance $P$. The probability measure $\mathbb{P}_{\mathcal{A}P}$ satisfies the following two properties,
\begin{enumerate}
    \item the probability of selecting arm $a$ at time $t$ is dictated only by the algorithm $\mathcal{A}$, and we denote the probability by $\mathcal{A}(a|\mathcal{H}_{t-1})$.
    \item the reward distribution of arm $a_t$, denoted by $\mathcal{X}_{a_t}$, is independent of the previous observed history $\mathcal{H}_{t-1}$.
\end{enumerate}

Hence, for any observed history $\mathcal{H}_T$, we have
\begin{equation}\label{eq-measure}
    \mathbb{P}^T_{\mathcal{A}P}\triangleq\mathbb{P}_{\mathcal{A}P}(\mathcal{H}_T)=\prod_{t=1}^T\mathcal{A}(a_t|\mathcal{H}_{t-1})\mathcal{X}_{a_t}(x_t).
\end{equation}

The next lemma states the KL-divergence decomposition for canonical bandit framework. Intuitively, by the decomposition, we separate the effect of the algorithm and the reward generation.

\begin{lemma}\label{KLDecomp}
(KL-divergence Decomposition). Given a bandit algorithm $\mathcal{A}$, two distinct instances $P_1,P_2$ and a probability measure $\mathbb{P}_{\mathcal{A}P}$ satisfying (\ref{eq-measure}). Then
\begin{equation}\label{eq-kldecom}
    {\rm KL}\left(\mathbb{P}_{\mathcal{A}P_1}^{T} \| \mathbb{P}_{\mathcal{A}P_2}^{T}\right)=\sum_{t=1}^{T} \mathbb{E}_{\mathcal{A}P_1}\left[{\rm KL}\left(\mathcal{A}\left(a_{t} \mid \mathcal{H}_{t}, P_{1}\right) \| \mathcal{A}\left(a_{t} \mid \mathcal{H}_{t}, P_{2}\right)\right)\right]+\sum_{a=1}^{K} \mathbb{E}_{\mathcal{A}P_1}\left[N_{a}(T)\right] {\rm KL}\left(\mathcal{X}_a^1 \|\mathcal{X}_a^2\right).
\end{equation}
where we use $\mathcal{X}_a^1$ and $\mathcal{X}_a^2$ to represent the reward distributions of arm $a$ in instance $P_1$ and $P_2$ respectively, and $N_a(T)$ is the times of pulling arm $a$ among the $T$ rounds.
\end{lemma}
For locally differentially private bandit algorithms, the first term on the LHS of (\ref{eq-kldecom}) vanishes since given the same history $\mathcal{H}_{t}$, $A\left(a_{t} \mid \mathcal{H}_{t}, P_{1}\right)$ and $A\left(a_{t} \mid \mathcal{H}_{t}, P_{2}\right)$ should be the same as they depends only on the internal randomness of the algorithm $\mathcal{A}$. The following lemma is about the locally private KL-divergence decomposition.

\begin{lemma}\label{LDPKLDecomp}
(Locally Private KL-divergence Decomposition). If the reward generation process is $\epsilon$-local differentially private for both the instance $P_1$ and $P_2$,  Then we have
\begin{equation}
    {\rm KL}\left(\mathbb{P}_{\mathcal{A}P_1}^{T} \| \mathbb{P}_{\mathcal{A}P_2}^{T}\right) \leq 2 \min \left\{4, e^{2 \epsilon}\right\}\left(e^{\epsilon}-1\right)^{2} \sum_{a=1}^{K} \mathbb{E}_{\mathcal{A}P_1}\left[N_{a}(T)\right] {\rm KL}\left(\mathcal{X}^1_{a} \| \mathcal{X}^2_{a}\right).
\end{equation}
\end{lemma}

Next, we prove Theorem \ref{LDPInstDepend}, Theorem \ref{LDPKIDLB} and Theorem \ref{LDPIndepLower}. We use $\mathcal{R}^\mathcal{A}_{T, P}$ to represent the regret of algorithm $\mathcal{A}$ under the instance $P$ with total $T$ rounds.

\begin{theorem}\label{LDPInstDepend}
(2-Armed LDP Instance-dependent Lower Bound). There exists a heavy-tailed two-armed  bandit instance with $u\le 1$ in (\ref{eq:2}) and $\Delta\triangleq\mu_1-\mu_2 \in (0,\frac{1}{5})$, such that for any $\epsilon$-LDP algorithm with $\epsilon \in (0,1]$ and regret $\leq o(T^\alpha)$ for any $\alpha > 0$, the regret satisfies
\[
\liminf _{T \rightarrow \infty} \frac{\mathcal{R}_T}{\log T}\geq  \Omega \left(\frac{1}{\epsilon^2 \Delta^{\frac{1}{v}}} \right).
\]
\end{theorem}
\begin{proof}[{of Theorem \ref{LDPInstDepend}}]
Consider the following instance $\bar P_1$:
the distribution of the first arm $a_1$ is $$\nu_{1}=\left(1-\frac{\gamma^{1+v}}{2}\right) \delta_{0}+\frac{\gamma^{1+v}}{2} \delta_{1 / \gamma}$$ with $\gamma=(5 \Delta)^{\frac{1}{v}} (\Delta \in \left(0,\frac{1}{5} \right))$, and the distribution of the second arm $a_2$ is $$\nu_{2}=\left(1-\left(\frac{\gamma^{1+v}}{2}-\Delta \gamma \right)\right) \delta_{0}+\left(\frac{\gamma^{1+v}}{2}-\Delta \gamma \right) \delta_{1 / \gamma}.$$ Thus, $$\begin{aligned}
 \mathbb{E}[\nu_1]&=\frac{5}{2}\Delta,  &u(\nu_1)&= \left(\frac{1}{\gamma}\right)^{1+v} \cdot \frac{\gamma^{1+v}}{2} =\frac{1}{2}<1,\\  \mathbb{E}[\nu_2]&=\frac{3}{2}\Delta,   &u(\nu_2)&= \left(\frac{1}{\gamma}\right)^{1+v} \cdot \left(\frac{\gamma^{1+v}}{2}-\Delta \gamma \right)<1.
\end{aligned}$$

Suppose we have another instance $\bar P_2 $: the distribution of the first arm $\nu_1^\prime$ is the same as $\nu_1$, and the distribution of the second arm is $$\nu_{2}^\prime=\left(1-\left(\frac{\gamma^{1+v}}{2}+\Delta \gamma \right)\right) \delta_{0}+\left(\frac{\gamma^{1+v}}{2}+\Delta \gamma \right) \delta_{1 / \gamma}.$$ Then $$\mathbb{E}[\nu_{2}^\prime]=\frac{7}{2}\Delta ,\ \ \  u(\nu_{2}^\prime)= \left(\frac{1}{\gamma}\right)^{1+v}\cdot \left(\frac{\gamma^{1+v}}{2}+\Delta \gamma \right) = \frac{1}{2}+\frac{\Delta}{\gamma^v}=\frac{7}{10}.$$

Since $\mathcal{R}_{T}=\sum \limits _{a:\Delta_a>0} \Delta_{a} \mathbb{E}\left[N_{a}(T)\right]$, we have
$$\mathbb{E}[\mathcal{R}^\mathcal{A}_{T, \bar{P}_1}] \geq \mathbb{P}^T_{\mathcal{A}\bar{P_1}}\left(N_2(T)\geq \frac{T}{2} \right) \cdot \left(\frac{5}{2}\Delta-\frac{3}{2}\Delta \right)\cdot\frac{T}{2}=\frac{T\Delta}{2} \cdot \mathbb{P}^T_{\mathcal{A}\bar{P}_{1}}\left(N_2(T)\geq \frac{T}{2}\right),$$
$$\mathbb{E}[\mathcal{R}^\mathcal{A}_{T, \bar{P}_2}] \geq \mathbb{P}^T_{\mathcal{A}\bar{P_2}}\left(N_2(T)\leq \frac{T}{2} \right) \cdot \left(\frac{7}{2}\Delta-\frac{5}{2}\Delta \right)\cdot\frac{T}{2}=\frac{T\Delta}{2} \cdot \mathbb{P}^T_{\mathcal{A}\bar{P}_{2}}\left(N_2(T)\leq \frac{T}{2}\right).$$

Thus, by Bretagnolle-Huber inequality~\citep[Theorem 14.2]{lattimore2020bandit}, we obtain
\begin{equation*}
\begin{aligned}
\mathbb{E}[\mathcal{R}^\mathcal{A}_{T, \bar{P}_1}]+\mathbb{E}[\mathcal{R}^\mathcal{A}_{T, \bar{P}_2}] &\geq \frac{T\Delta}{2}\left[\mathbb{P}^T_{\mathcal{A}\bar{P}_{1}}\left(N_2(T)\geq \frac{T}{2}\right)+\mathbb{P}^T_{\mathcal{A}\bar{P}_{2}}\left(N_2(T)\leq \frac{T}{2}\right) \right]\\
&\geq \frac{T\Delta}{4} \cdot \exp{\left(-{\rm KL}(\mathbb{P}^T_{\mathcal{A}\bar{P}_{1}}\|\mathbb{P}^T_{\mathcal{A}\bar{P}_{2}})\right)}.
\end{aligned}
\end{equation*}
By Lemma \ref{LDPKLDecomp}, we have ${\rm KL}(\mathbb{P}^T_{\mathcal{A}\bar{P}_{1}}\|\mathbb{P}^T_{\mathcal{A}\bar{P}_{2}}) \leq 8\left(e^\epsilon-1 \right)^2 \cdot \mathbb{E}_{\mathcal{A} \bar P_1}[N_2(T)]\cdot {\rm KL}(\nu_2 \|\nu_2^\prime)$. Thus,
\[
    \mathbb{E}[\mathcal{R}^\mathcal{A}_{T, \bar{P}_1}]+\mathbb{E}[\mathcal{R}^\mathcal{A}_{T, \bar{P}_2}] \geq \frac{T \Delta}{4}\exp{\left(-8\left(e^\epsilon-1 \right)^2 \cdot \mathbb{E}_{\mathcal{A}\bar P_1}[N_2(T)]\cdot {\rm KL}(\nu_2 \|\nu_2^\prime) \right)}.
\]
Then, we obtain
\begin{equation}\label{eq-EN2T}
    \mathbb{E}_{\mathcal{A} \bar P_1}[N_2(T)] \geq \frac{\log \frac{T\Delta}{4}-\log\left(\mathbb{E}[\mathcal{R}^\mathcal{A}_{T, \bar{P}_1}]+\mathbb{E}[\mathcal{R}^\mathcal{A}_{T, \bar{P}_2}] \right)}{8\left(e^\epsilon-1 \right)^2 \cdot {\rm KL}(\nu_2 \|\nu_2^\prime)}\geq \frac{\log \frac{T\Delta}{4}-2\alpha \log T}{8\epsilon^2 \cdot \rm KL(\nu_2 \|\nu_2^\prime)},
\end{equation}
where the last inequality is due to the assumption of sub-linear regret $\mathcal{R}_T \leq o(T^\alpha)$ and the fact that $e^\epsilon-1 \approx \epsilon$ when $\epsilon$ is small.

By using ${\rm KL}\left(\operatorname{Ber}(p) \|  \operatorname{Ber}(q)\right) \leq \frac{(p-q)^{2}}{q(1-q)}$, we obtain
\begin{align*}
    {\rm KL}(\nu_2 \|\nu_2^\prime) &=\rm{KL} \left(\rm{Ber}\left(\frac{\gamma^{1+\it v}}{2}-\Delta\gamma\right)\Big\| \rm{Ber} \left(\frac{\gamma^{1+\it v}}{2}+\Delta\gamma\right)\right)\\
    &\leq \frac{(2\Delta\gamma)^2}{\left(\frac{\gamma^{1+\it v}}{2}+\Delta\gamma\right)\cdot \left(1-\left(\frac{\gamma^{1+\it v}}{2}+\Delta\gamma\right)\right)}.
\end{align*}

Note that $\gamma=(5\Delta)^\frac{1}{v}$, we get
\begin{align*}
    \rm{KL}(\nu_2 \|\nu_2^\prime) & \leq \frac{\left(2\cdot 5^{\frac{1}{\it v}} \cdot \Delta^{\frac{1+\it v}{\it v}}\right)^2}{\left(\frac{5^{\frac{1+\it v}{\it v}}\cdot \Delta^{\frac{1+\it v}{\it v}}}{2}+5^{\frac{1}{\it v}}\cdot \Delta^{\frac{1+\it v}{\it v}}\right)\left(1-\frac{7}{2}\cdot 5^{\frac{1}{\it v}}\Delta^{\frac{1+\it v}{\it v}}\right)} \\
    &\leq \frac{4\cdot 5^{\frac{1}{\it v}} \cdot \Delta^{\frac{1+\it v}{\it v}}}{\frac{7}{2}\left(1-\frac{7}{2}\cdot 5^{\frac{1}{\it v}}\Delta^{\frac{1+\it v}{\it v}}\right)}\leq C \cdot 5^{\frac{1}{\it v}} \cdot \Delta^{\frac{1+\it v}{\it v}},
\end{align*}
where $C$ is some constant and the last inequality holds when $\Delta\in(0,\frac{1}{5})$ is sufficiently small.

Thus, according to (\ref{eq-EN2T}),
\[
    \liminf _{T \rightarrow \infty} \frac{\mathbb{E}_{\mathcal{A} \bar P_1}[N_2(T)]}{\log T} \geq \Omega\left(\frac{1}{5^{\frac{1}{\it v}} \Delta^{\frac{1+\it v}{\it v}}\epsilon^2}\right),
\]
then,
\[
  \liminf _{T \rightarrow \infty} \frac{\mathcal{R}^\mathcal{A}_{T, \bar{P}_1}}{\log T} \geq \liminf _{T \rightarrow \infty} \frac{\mathbb{E}_{\mathcal{A} \bar P_1}[N_2(T)]}{\log T}\cdot \Delta \ge \Omega\left(\frac{1}{5^{\frac{1}{\it v}} \Delta^{\frac{1}{\it v}}\epsilon^2}\right).
\]
\end{proof}

\begin{proof}[{of Theorem \ref{LDPKIDLB}}]
We focus on the $K$ arms with mean reward satisfying $\frac{1}{2}\ge\mu_1\ge\cdots\ge\mu_K$ and $\frac{1}{5}\mu_a\le\Delta_a\le\frac{1}{2}\mu_a$. Consider the instance $\bar{P}$ where the distribution for each arm $a\in[K]$ is
\[
    \nu_a=\left(1-\frac{s_a^{1+v}}{2}\right)\delta_0+\frac{s_a^{1+v}}{2}\delta_{1/s_a},
\]
where $s_a=(2\mu_a)^{\frac{1}{v}}$. It is easy to verify for each $a\in[K]$ that
\[
    \mathbb{E}[\nu_a]=\mu_a, u(\nu_a)=\frac{1}{2}<1.
\]
Then consider another instance $\bar{Q}_a$, where the reward distribution of any arm $a^\prime\neq a$ remains unchanged and the reward distribution of $a$ becomes
\[
    \nu_a^\prime=[1-(\frac{s_a^{1+v}}{2}+2\Delta_as_a)]\delta_0+(\frac{s_a^{1+v}}{2}+2\Delta_as_a)\delta_{1/s_a}.
\]
Note that, $\frac{s_a^{1+v}}{2}+2\Delta_as_a=2^{\frac{1}{v}}\mu_a^{\frac{1+v}{v}}+2^{\frac{1+v}{v}}\Delta_a\mu_a^{\frac{1}{v}}\le\mu_a+2\Delta_a=\mu_1+\Delta_a\le1$, where the first inequality is due to $\mu_a\le\frac{1}{2}$, hence the postulated $\nu_a^\prime$ is reasonable.

For $\nu_a^\prime$, we have $\mathbb{E}[\nu_a^\prime]=\mu_a+2\Delta_a=\mu_1+\Delta_a$ and $u(\nu_a^\prime)=\frac{1}{2}+\frac{\Delta_a}{\mu_a}\le1$, where the inequality is due to $\Delta_a\le\frac{1}{2}\mu_a$.

Since $\mathcal{R}_T=\sum\limits_{a:\Delta_a>0}\Delta_a\mathbb{E}[N_a(T)]$, we have
\begin{equation*}
    \mathbb{E}[\mathcal{R}^\mathcal{A}_{T,\bar{P}}]\ge\mathbb{P}^T_{\mathcal{A}\bar{P}}\left(N_a(T)\ge\frac{T}{2}\right)\cdot\Delta_a\cdot\frac{T}{2}=\frac{T}{2}\Delta_a\mathbb{P}^T_{\mathcal{A}\bar{P}}\left(N_a(T)>\frac{T}{2}\right),
\end{equation*}
\begin{equation*}
    \mathbb{E}[\mathcal{R}^\mathcal{A}_{T,\bar{Q}_a}]\ge\mathbb{P}^T_{\mathcal{A}\bar{Q}_a}\left(N_a(T)\le\frac{T}{2}\right)\cdot\Delta_a\cdot\frac{T}{2}=\frac{T}{2}\Delta_a\mathbb{P}^T_{\mathcal{A}\bar{Q}_a}\left(N_a(T)\le\frac{T}{2}\right).
\end{equation*}
By Bretagnolle-Huber inequality~\citep[Theorem 14.2]{lattimore2020bandit}, we obtain
\begin{align*}
\mathbb{E}[\mathcal{R}^\mathcal{A}_{T, \bar{P}}]+\mathbb{E}[\mathcal{R}^\mathcal{A}_{T, \bar{Q}_a}] &\ge \frac{T\Delta_a}{2}\left[\mathbb{P}^T_{\mathcal{A}\bar{P}}\left(N_a(T)\geq \frac{T}{2}\right)+\mathbb{P}^T_{\mathcal{A}\bar{Q}_{a}}\left(N_a(T)\leq \frac{T}{2}\right) \right]\\
&\geq \frac{T\Delta_a}{4} \cdot \exp{\left(-{\rm KL}(\mathbb{P}^T_{\mathcal{A}\bar{P}}\|\mathbb{P}^T_{\mathcal{A}\bar{Q}_{a}})\right)}.
\end{align*}\
Due to Lemma \ref{LDPKLDecomp}, we have ${\rm KL}(\mathbb{P}^T_{\mathcal{A}\bar{P}}\|\mathbb{P}^T_{\mathcal{A}\bar{Q}_{a}}) \leq 8\left(e^\epsilon-1 \right)^2 \cdot \mathbb{E}_{\mathcal{A} \bar P}[N_a(T)]\cdot {\rm KL}(\nu_a \|\nu_a^\prime)$. Thus we obtain that
\begin{equation*}
    \mathbb{E}[\mathcal{R}^\mathcal{A}_{T, \bar{P}}]+\mathbb{E}[\mathcal{R}^\mathcal{A}_{T, \bar{Q}_a}] \ge \frac{T\Delta_a}{4} \cdot \exp{\left(-8\left(e^\epsilon-1 \right)^2 \cdot \mathbb{E}_{\mathcal{A} \bar P}[N_a(T)]\cdot {\rm KL}(\nu_a \|\nu_a^\prime)\right)},
\end{equation*}
which gives that
\begin{equation}\label{eq-ENAT}
    \mathbb{E}_{\mathcal{A}\bar{P}}[N_a(T)]\ge\frac{\log\frac{T\Delta_a}{4}-\log(\mathbb{E}[\mathcal{R}^\mathcal{A}_{T,\bar{P}}]+\mathbb{E}[\mathcal{R}^\mathcal{A}_{T,\bar{Q}_a}])}{8(e^\epsilon-1)^2\cdot {\rm KL}(\nu_a\|\nu_a^\prime)}\ge\frac{\log\frac{T\Delta_a}{4}-2\alpha\log T}{8\epsilon^2\cdot {\rm KL}(\nu_a\|\nu_a^\prime)},
\end{equation}
where the last inequality is due to the assumption of sub-linear regret $\mathcal{R}_T \leq o(T^\alpha)$ and the fact that $e^\epsilon-1 \approx \epsilon$ when $\epsilon$ is small.

By using the fact that ${\rm KL}\left(\operatorname{Ber}(p) \|  \operatorname{Ber}(q)\right) \leq \frac{(p-q)^{2}}{q(1-q)}$, we can obtain that
\begin{align*}
    {\rm KL}(\nu_a \|\nu_a^\prime) &={\rm KL} \left({\rm Ber}\left(\frac{s_a^{1+v}}{2}\right)\Big\| {\rm Ber} \left(\frac{s_a^{1+v}}{2}+2\Delta_a s_a\right)\right)\\
    &\leq \frac{(2\Delta\gamma)^2}{\left(\frac{\gamma^{1+\it v}}{2}+\Delta\gamma\right)\cdot \left(1-\left(\frac{\gamma^{1+\it v}}{2}+\Delta\gamma\right)\right)}\\
    &\le \frac{25^\frac{1}{v}\Delta_a^{\frac{1+v}{v}}}{1-7\cdot10^{\frac{1}{v}}\Delta_a^{\frac{1+v}{v}}}\le C\cdot(25)^{\frac{1}{v}}\Delta_a^{\frac{1+v}{v}},
\end{align*}
where $C$ is some constant and the last inequality holds since $\frac{1}{5}\mu_a \le \Delta_a\le\frac{1}{2}\mu_a$ and $\mu_a\le\frac{1}{2}$.
Thus, according to (\ref{eq-ENAT}), we have
\begin{equation*}
    \liminf _{T \rightarrow \infty} \frac{\mathbb{E}_{\mathcal{A} \bar P}[N_a(T)]}{\log T}\ge\Omega\left(\frac{1}{\epsilon^2\Delta_a^{\frac{1+v}{v}}}\right),
\end{equation*}
then,
\begin{equation*}
    \liminf _{T \rightarrow \infty} \frac{\mathcal{R}^\mathcal{A}_{T, \bar{P}}}{\log T} \geq \liminf _{T \rightarrow \infty}
    \sum_{a:\Delta_a>0}\frac{\mathbb{E}_{\mathcal{A} \bar P}[N_a(T)]\cdot \Delta_a}{\log T} \ge \Omega \left(\frac{1}{\epsilon^2}\sum_{\Delta_a>0}(\frac{1}{\Delta_a})^{\frac{1}{v}} \right).
\end{equation*}
\end{proof}

\begin{proof}[{of Theorem \ref{LDPIndepLower}}]
We first define the instance $\bar{P}_1$. In $\bar{P}_1$, the optimal arm (denoted by $a_1$) follows the reward distribution
\[
    \nu_1=\left(1-\frac{\gamma^{1+v}}{2}\right)\delta_0 +\frac{\gamma^{1+v}}{2}\delta_{1/\gamma},
\]
where $\gamma=(5\Delta)^{\frac{1}{v}} (\Delta \in \left(0,\frac{1}{5} \right))$. Note that $\mathbb{E}[\nu_1]=\frac{5}{2}\Delta,u(\nu_1)=\frac{1}{2}$.

 Any other sub-optimal arm $a\neq a_1$ in $\bar{P}_1$ follows the same reward distribution
 \[
    \nu_a=\left(1-\frac{\gamma^{1+v}}{2}+\Delta\gamma\right)\delta_0 +\left(\frac{\gamma^{1+v}}{2}-\Delta\gamma\right)\delta_{1/\gamma}.
\]
Note that for all $a\neq a_1$ $\mathbb{E}[\nu_a]=\frac{3}{2}\Delta$, $u(\nu_a)=\frac{1}{2}-\frac{1}{5}=\frac{3}{10}<1$.
We denote the corresponding locally private reward distribution for each arm $a\in[K]$ as $\bar \nu_a$.

For algorithm $\mathcal{A}$ and instance $\bar P_1$, we denote $
    i={\arg \min} _{a\in \{2,\cdots,K\}}\mathbb{E}_{\mathcal{A}\bar{P}_1}[N_a(T)].$
Thus, $\mathbb{E}_{\mathcal{A}\bar{P}_1}[N_i(T)]\leq\frac{T}{K-1}$.

Now, consider another instance $\bar{P}_i$ where $\nu_1,\cdots,\nu_K$ are the same as those in $\bar P_1$ except the $i$-th arm such that
\[
    \nu_i^\prime=\left(1-\frac{\gamma^{1+v}}{2}-\Delta\gamma\right)\delta_0 +\left(\frac{\gamma^{1+v}}{2}+\Delta\gamma\right)\delta_{1/\gamma}.
\]
Note that now $\mathbb{E}[\nu_i^\prime]=\frac{7}{2}\Delta$, $u(\nu_i^\prime)=\frac{7}{10}<1$. Similarly, we denote the corresponding locally private reward distribution for arm $i$ as $\bar {\nu}_i^\prime$.

Thus, $$\mathbb{E}[\mathcal{R}^\mathcal{A}_{T, \bar{P}_1}]\geq \mathbb{P}^T_{\mathcal{A}\bar P_1} \left[ N_i(T) \geq \frac{T}{2} \right]\frac{T}{2}\Delta,$$
$$\mathbb{E}[\mathcal{R}^\mathcal{A}_{T, \bar{P}_i}]\geq \mathbb{P}^T_{\mathcal{A}\bar P_i} \left[ N_i(T) \leq \frac{T}{2} \right]\frac{T}{2}\Delta.$$
Thus by Bretagnolle-Huber inequality~\citep[Theorem 14.2]{lattimore2020bandit} and Lemma \ref{LDPKLDecomp} we have,
\begin{align*}
    \mathbb{E}[\mathcal{R}^\mathcal{A}_{T, \bar{P}_1}]+\mathbb{E}[\mathcal{R}^\mathcal{A}_{T, \bar{P}_i}] &\geq \frac{T\Delta}{4}\exp{\left(- {\rm KL}(\mathbb{P}^T_{\mathcal{A}\bar{P}_1} \| \mathbb{P}^T_{\mathcal{A}\bar{P}_i})\right)} \\
    &\ge \frac{T\Delta}{4}\exp{\left(- \mathbb{E}_{\mathcal{A}\bar{P}_1}[N_i(T)] \cdot \rm KL(\bar{\nu}_i \| \bar{\nu}_i^\prime)\right)}\\
    &\geq \frac{T\Delta}{4}\exp{\left(-8(\rm{e}^\epsilon-1)^2 \cdot \mathbb{E}_{\mathcal{A}\bar{P}_1}[N_i(T)] \cdot \rm KL(\nu_i \| \nu_i^\prime)\right)}.
\end{align*}

Since $$\rm KL(\nu_i \| \nu_i^\prime) \leq \frac{(2\Delta\gamma)^2}{\left(\frac{\gamma^{1+\it v}}{2}+\Delta\gamma\right)\left(1-\left(\frac{\gamma^{1+\it v}}{2}+\Delta\gamma\right)\right)} \leq C\cdot 5^{\frac{1}{\it v}}\cdot \Delta^{\frac{1+\it v}{\it v}},$$
for some constant $C>0$ and $\Delta$ is sufficiently small.

We obtain
\[
    \mathbb{E}[\mathcal{R}^\mathcal{A}_{T, \bar{P}_1}]\geq \frac{T\Delta}{8}\exp{\left(-8\epsilon^2\cdot\frac{T}{K-1}\cdot C\cdot 5^{\frac{1}{v}}\cdot \Delta^{\frac{1+v}{v}}\right)}.
\]
Taking $\Delta=\left(\frac{K}{T\epsilon^2}\right)^{\frac{v}{1+v}}$, we get the result
$$\mathbb{E}[\mathcal{R}^\mathcal{A}_{T, \bar{P}_1}]\geq \Omega\left(T^{\frac{1}{1+v}}\left(\frac{K}{\epsilon^2}\right)^{\frac{v}{1+v}}\right). $$
\end{proof}

\section{OMITTED EXPERIMENTAL RESULTS FOR SECTION~\ref{sec:experiments}}
\label{sec:appendix-experiments}
\begin{figure}[!htb]
    \centering
    \subcaptionbox{$v=0.5, \epsilon=0.5$}{\includegraphics[width=.245\linewidth]{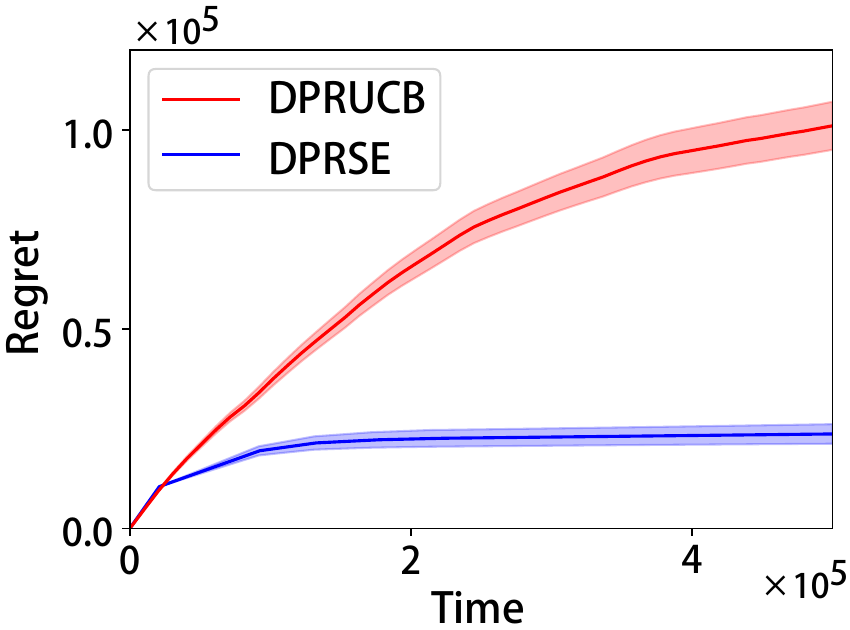}}
    \subcaptionbox{$v=0.5, \epsilon=1.0$}{\includegraphics[width=.245\linewidth]{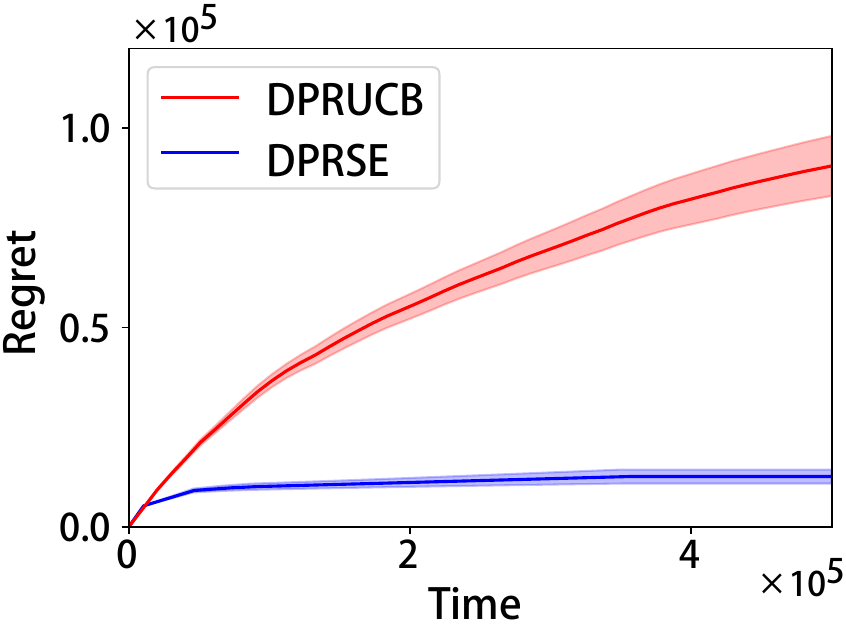}}
    \subcaptionbox{$v=0.9, \epsilon=0.5$}{\includegraphics[width=.245\linewidth]{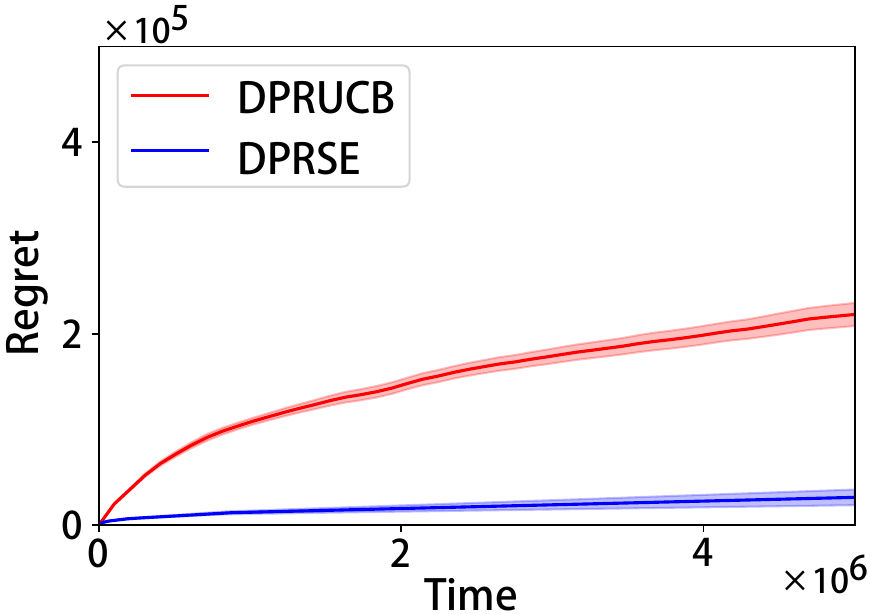}}
    \subcaptionbox{$v=0.9, \epsilon=1.0$}{\includegraphics[width=.245\linewidth]{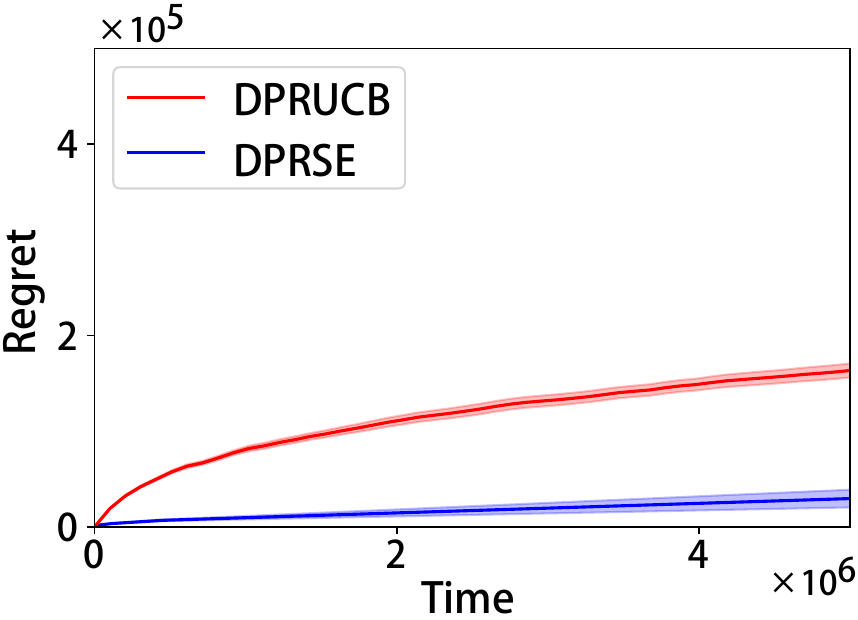}}
    \caption{DP Setting 2 ($S_2$)}
    \label{fig:DP2}
\end{figure}

\begin{figure}[!htb]
    \centering
    \subcaptionbox{$v=0.5, \epsilon=0.5$}{\includegraphics[width=.245\linewidth]{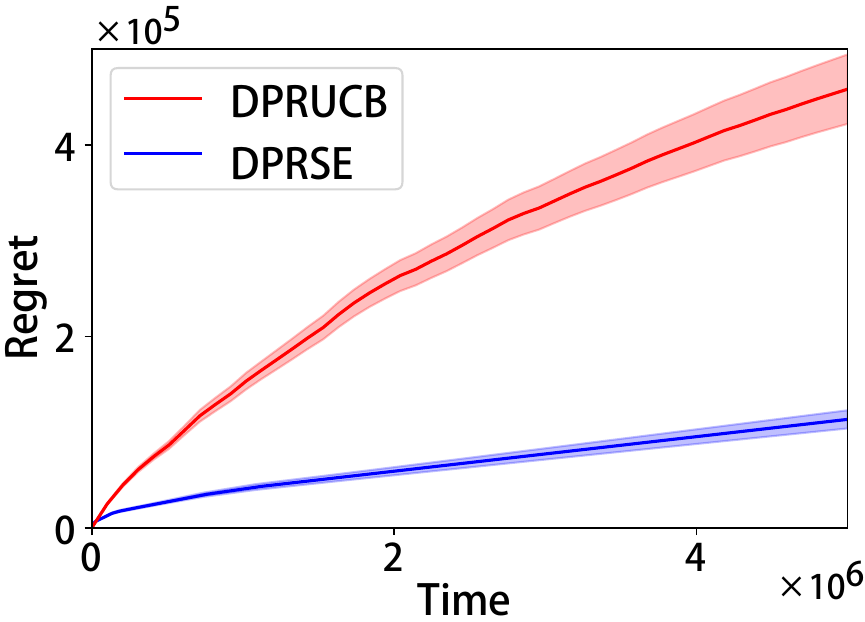}}
    \subcaptionbox{$v=0.5, \epsilon=1.0$}{\includegraphics[width=.245\linewidth]{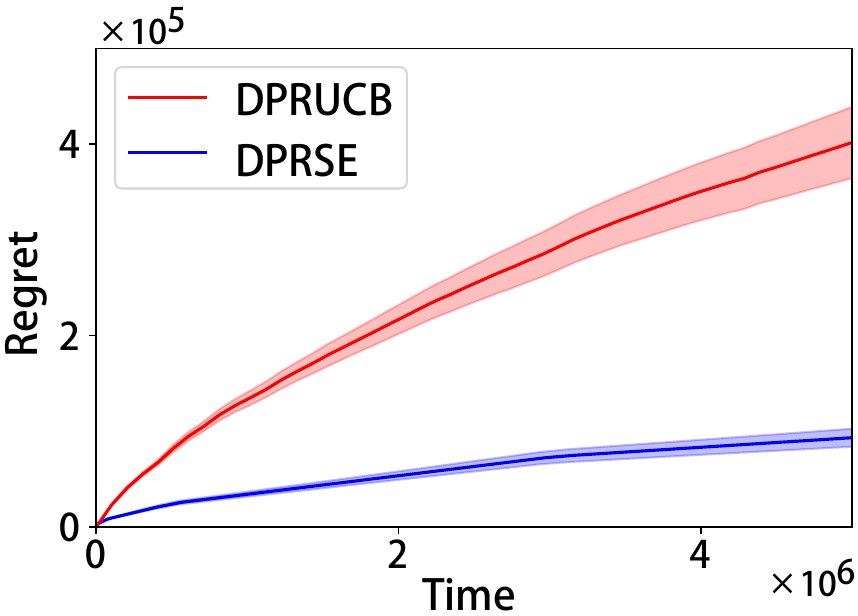}}
    \subcaptionbox{$v=0.9, \epsilon=0.5$}{\includegraphics[width=.245\linewidth]{DPS3F3.pdf}}
    \subcaptionbox{$v=0.9, \epsilon=1.0$}{\includegraphics[width=.245\linewidth]{DPS3F4.pdf}}
    \caption{DP Setting 3 ($S_3$)}
    \label{fig:DP3}
\end{figure}

\end{document}